\documentclass{article}

\usepackage{microtype}
\usepackage{graphicx}
\usepackage{booktabs} 
\usepackage{amsmath}
\usepackage{amssymb}
\usepackage{amsthm}
\usepackage{mathtools}
\usepackage{subcaption}
\usepackage{hyperref}

\theoremstyle{definition}
\newtheorem{definition}{Definition}[section]
\newtheorem{theorem}{Theorem}
\newtheorem{lemma}{Lemma}
\newcommand{\norm}[1]{\left\lVert#1\right\rVert}
\DeclareMathOperator*{\argmax}{arg\,max}

\DeclareMathOperator{\EX}{\mathbb{E}}
\usepackage{hyperref}

\usepackage[accepted]{icml2021}

\icmltitlerunning{Generalizable Episodic Memory for Deep Reinforcement Learning }

\begin{document}

\twocolumn[
\icmltitle{Generalizable Episodic Memory for Deep Reinforcement Learning }



\icmlsetsymbol{equal}{*}

\begin{icmlauthorlist}
\icmlauthor{Hao Hu}{IIIS}
\icmlauthor{Jianing Ye}{PKU}
\icmlauthor{Guangxiang Zhu}{IIIS}
\icmlauthor{Zhizhou Ren}{UIUC}
\icmlauthor{Chongjie Zhang}{IIIS}

\end{icmlauthorlist}

\icmlaffiliation{IIIS}{The Institute for Interdisciplinary Information Sciences, Tsinghua University, Beijing, China}
\icmlaffiliation{PKU}{Peking University, Beijing, China}
\icmlaffiliation{UIUC}{University of Illinois at Urbana-Champaign, IL, USA}

\icmlcorrespondingauthor{Chongjie Zhang}{chongjiezhang@mail.tsinghua.edu.cn}


\icmlkeywords{Generalizable Episodic Memory, ICML}

\vskip 0.3in
]



\printAffiliationsAndNotice{}

\begin{abstract}
Episodic memory-based methods can rapidly latch onto past successful strategies by a non-parametric memory and improve sample efficiency of traditional reinforcement learning. However, little effort is put into the continuous domain, where a state is never visited twice, and previous episodic methods fail to efficiently aggregate experience across trajectories. To address this problem, we propose \textbf{G}eneralizable \textbf{E}pisodic \textbf{M}emory (\textbf{GEM}), which effectively organizes the state-action values of episodic memory in a generalizable manner and supports implicit planning on memorized trajectories. GEM utilizes a double estimator to reduce the overestimation bias induced by value propagation in the planning process. Empirical evaluation shows that our method significantly outperforms existing trajectory-based methods on various MuJoCo continuous control tasks. To further show the general applicability, we evaluate our method on Atari games with discrete action space, which also shows a significant improvement over baseline algorithms.
\end{abstract}

\section{Introduction}
Deep reinforcement learning (RL) has been tremendously successful in various domains, like classic games \citep{AlphaGo}, video games \citep{DQN}, and robotics \citep{DDPG}. However, it still suffers from high sample complexity, especially compared to human learning \citep{DBLP:conf/aaaiss/TsividisPXTG17}. One significant deficit comes from gradient-based bootstrapping, which is incremental and usually very slow \citep{fastslow}.\par

\begin{figure}[H]
    \vspace{0.2in}
    \includegraphics[width=8cm]{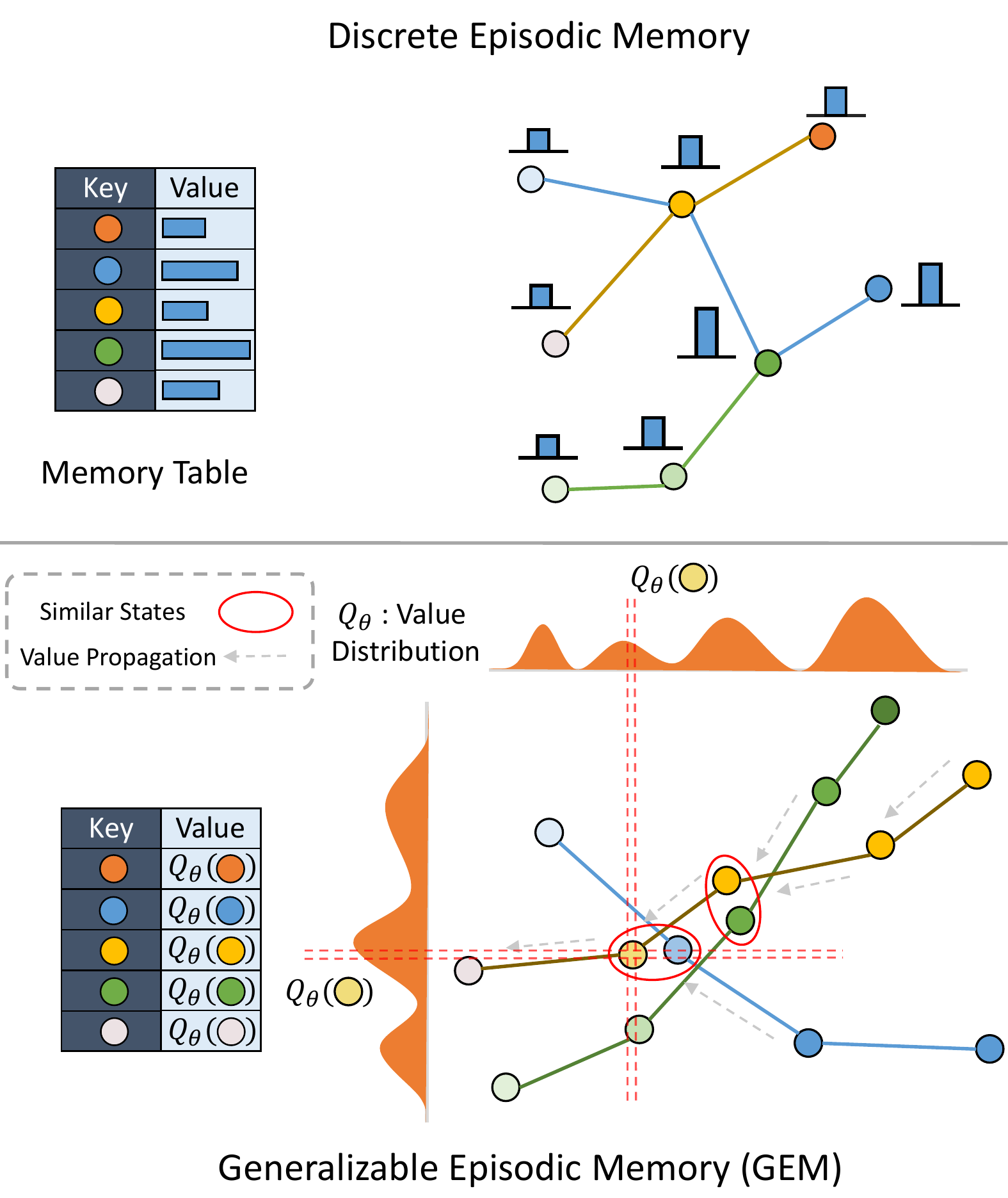}
    \caption{Illustration of our basic idea. Conventional episodic memory usually uses a non-parametric schema to store every state-action pairs and updates their values when re-encountering the same events. Instead, generalizable episodic memory stores the $Q$-values for each event by a parametric network $\mathcal{M}_\theta$ so that each single event can generalize to their neighborhoods. We further perform value propagation along adjacent states to encourage information exchange between related events. }\label{illustration_2}
\end{figure}

Inspired by psychobiological studies of human's episodic memory \citep{sutherland1989configural,marr1991simple,lengyel2007hippocampal} and instance-based decision theory \citep{gilboa1995case}, episodic reinforcement learning  \citep{MFEC,NEC,EMDQN,EVA,associative} presents a non-parametric or semi-parametric framework that fast retrieves past successful strategies to improve sample efficiency. Episodic memory stores past best returns, and the agents can act accordingly to repeat the best outcomes without gradient-based learning.\par

However, most existing methods update episodic memory only by exactly re-encountered events, leaving the generalization ability aside. As Heraclitus, the Greek philosopher, once said, ``No man ever steps in the same river twice.'' \citep{kahn1981art} Similarly, for an RL agent acting in continuous action and state spaces, the same state-action pair can hardly be observed twice. However, humans can connect and retrospect from similar experiences of different times, with no need to re-encounter the same event \citep{shohamy2008integrating}. Inspired by this human's ability to learn from generalization, we propose \textbf{G}eneralizable \textbf{E}pisodic \textbf{M}emory (\textbf{GEM}), a novel framework that integrates the generalization ability of neural networks and the fast retrieval manner of episodic memory.\par

We use Figure \ref{illustration_2} to illustrate our basic idea. Traditional discrete episodic control methods usually build a non-parametric slot-based memory table to store state-value pairs of historical experiences. In the discrete domain, states can be re-encountered many times. This re-encountering enables traditional methods to aggregate trajectories by directly taking maximum among all historical returns. However, this is not feasible in environments with high dimensional states and action space, especially in the continuous domain, where an agent will never visit the same state twice. On the contrary, GEM learns a virtual memory table memorized by deep neural networks to aggregate similar state-action pairs that essentially have the same nature. This virtual table naturally makes continuous state-action pairs generalizable by reconstructing experiences' latent topological structure with neural networks. This memory organization enables planning across different trajectories, and GEM uses this capability to do implicit planning by performing value propagation along trajectories saved in the memory and calculating the best sequence over all possible real and counterfactual combinatorial trajectories. \par

We further propose to use twin networks to reduce the overestimation of these aggregated returns. Merely taking maximum among all possible plans leads to severe overestimation since overestimated values are preserved along the trajectory. Thus, we use two networks to separately calculate which trajectory has the maximal value and how large the maximum value is,  which shares the similar idea with double Q-learning \citep{DoubleQ}.\par

The main contribution of our proposed method is threefold:
1. We present a novel framework inspired by psychology, GEM, to build generalizable episodic memory and improve sample efficiency. GEM consists of a novel value planning scheme for the continuous domain and a twin back-propagation process to reduce overestimation along the trajectories;
2. We formally analyze the estimation and convergence properties of our proposed algorithm;
3. We empirically show the significant improvement of GEM over other baseline algorithms and its general applicability across continuous and discrete domains.
Besides, we analyze the effectiveness of each part of GEM by additional comparisons and ablation studies. 

\section{Preliminaries}
We consider a Markov Decision Process (MDP), defined by the tuple $\langle\mathcal{S},\mathcal{A},P,r,\gamma\rangle$, where $\mathcal{S}$ is the state space and $\mathcal{A}$ the action space. $P(\cdot|s,a) : \mathcal{S} \times \mathcal{A} \times \mathcal{S} \rightarrow \mathbb{R}$ denotes the transition distribution function and $r(s,a):\mathcal{S}\times\mathcal{A}\rightarrow \mathbb{R}$ the reward function. $\gamma\in[0,1)$ is the discount factor.\par
The goal of an RL agent is to learn a policy $\pi:\mathcal{S} \times \mathcal{A} \rightarrow \mathbb{R}$ which maximizes the expectation of a discounted cumulative reward, i.e.,
        \begin{equation*}
            J(\pi)=\EX_{\substack{s_0\sim\rho_0,a_t\sim\pi(\cdot|s_t),\\s_{t+1}\sim P(\cdot|s_t,a_t)}}\left[\sum_{t=0}^{\infty}\gamma^t r(s_t,a_t) \right],
        \end{equation*}
where $\rho_0$ denotes the distribution of initial states.

\subsection{Deterministic Policy Gradient}
In the context of continuous control, actor-critic architecture is widely used \citep{SuttonAC,PetersAC} to handle the continuous action space. In actor-critic framework, a critic $Q_\theta$ is used to estimate the action-value function $Q_\theta(s,a)\approx Q^\pi(s,a)$, while the parametric actor $\pi_\phi$ optimizes its policy using the policy gradient from parametric $Q_\theta$. When $\pi_\phi$ is a deterministic function, the gradient can be expressed as follows:
\begin{equation*}
    \nabla_\phi J(\pi_\phi)=\EX_{s\sim p_\pi}[\nabla_a Q_\theta(s,a)|_{a=\pi(s)}\nabla_\phi\pi_\phi(s)],
\end{equation*}
which is known as deterministic policy gradient theorem \citep{DPG}.\\
$Q_\theta$ can be learned from the TD-target derived from the Bellman Equation \citep{bellman}:
$$y=r(s,a)+\gamma Q_{\theta'}(s',a'),a'\sim\pi_\phi(s')$$

In the following section, when it does not lead to confusion, we use $Q_\theta(s)$ instead of $Q_\theta(s,\pi_\phi(s))$ for simplicity.\par 

\subsection{Double Q-Learning}
Q-learning update state-action values by the TD target \citep{TDLearning} :$$r+\gamma\max_a Q(s',a).$$ However, the max operator in the target can easily lead to overestimation, as discussed in \citet{DoubleQ}. Thus, They propose to estimate the maximum $Q$-values over all actions $\max_a Q(s,a)$ using the double estimator, i.e.,
\begin{equation*}
    Q_{\text{double}}(s)=Q^A(s,a^*_B), a^*_B=\argmax_a Q^B(s,a),
\end{equation*}
 where $Q^A,Q^B$ are two independent function approximators for $Q$-values.

\subsection{Model-Free Episodic Control}
Traditional deep RL approaches suffer from sample inefficiency because of slow gradient-based updates. Episodic control is proposed to speed up the learning process by a non-parametric episodic memory (EM). The key idea is to store good past experiences in a tabular-based non-parametric memory and rapidly latch onto past successful policies when encountering similar states, instead of waiting for many steps of optimization. When different experiences meet up at the same state-action pair $(s, a)$, model-free episodic control \citep[MFEC;][]{MFEC} aggregates the values of different trajectories by taking the maximum return $R$ among all these rollouts starting from the intersection $(s, a)$. Specifically, $Q^{EM}$ is updated by the following equation:
\begin{equation}
\label{equ:mfec2}
Q^{EM}(s,a)=
    \left\{
        \begin{aligned}
            &R,  &\text{if~} (s,a) \notin EM,\\
            &\max{\{R,Q^{EM}(s,a)\}}, &\text{otherwise.} \\
        \end{aligned}
        \right.
\end{equation}
At the execution time, MFEC selects the action according to the maximum $Q$-value of the current state. If there is no exact match of the state, MFEC performs a $k$-nearest-neighbors lookup to estimate the state-action values, i.e.,
        \begin{equation}
            \label{equ:mfec}
            \widehat{Q}^{EM}(s,a)=\left\{
                \begin{aligned}
                    &\frac{1}{k}\sum_{i=1}^k Q(s_i,a), &\mbox{if } (s,a) \notin Q^{EM},\\
                    &Q^{EM}(s,a), &\mbox{otherwise},  \\
                \end{aligned}
                \right.
        \end{equation}
where $s_i$ $(i=1,\cdots,k)$ are the $k$ nearest states from $s$.  

\section{Generalizable Episodic Memory}
\subsection{Overview}
\label{overview}
In generalizable episodic memory (GEM), we use a parametric network $\mathcal{M}_\theta$ to represent the virtual memory table with parameter $\theta$, which is learned from tabular memory $\mathcal{M}$. To leverage the generalization ability of $\mathcal{M}_\theta$, we enhance the returns stored in the tabular memory by propagating value estimates from ${\mathcal{M}}_\theta$ and real returns from $\mathcal{M}$ along trajectories in the memory and take the best value over all possible rollouts, as depicted in Equation~(\ref{singlebp}). Generalizable memory $\mathcal{M}_\theta$ is trained by performing regression toward this enhanced target and is then used to guide policy learning as well as to build the new target for GEM's learning.
\par
A significant issue for the procedure above is the overestimation induced from taking the best value along the trajectory. Overestimated values are preserved when doing back-propagation along the trajectory and hinder learning efficiency. To mitigate this issue, we propose to use twin networks for the back-propagation of value estimates. This twin back-propagation process uses the double estimator for estimating the maximum along trajectories, which resembles the idea of Double Q-learning \citep{DoubleQ} and is illustrated in detail in Section~\ref{TBN}. It is well known that vanilla reinforcement learning algorithms with function approximation already has a tendency to overestimate \citep{DoubleQ, TD3}, making the overestimation reduction even more critical. To solve this challenge, we propose to use additional techniques to make the value estimation from $\mathcal{M}_\theta$ more conservative, which is illustrated in Section~\ref{conservative}.\par
A formal description for GEM algorithm is shown in Algorithm \ref{alg:amc}. In the following sections, we use $Q_\theta$ to represent $Q_{\mathcal{M}_\theta}$ for simplicity. Our implementation of GEM is available at \url{https://github.com/MouseHu/GEM}.\par
\begin{algorithm}[tb]
    \caption{Generalizable Episodic Memory}
    \label{alg:amc}
    \begin{algorithmic}
        \STATE Initialize episodic memory(critic) networks $Q^{(1)}_\theta, Q^{(2)}_\theta$, and actor network $\pi_\phi$ with parameters $\theta_{(1)},\theta_{(2)},\phi$
        \STATE Initialize targets $\theta_{(1)}'\leftarrow\theta_{(1)},\theta_{(2)}'\leftarrow\theta_{(2)},\phi'\leftarrow\phi$
        \STATE Initialize episodic memory $\mathcal{M}$

        \FOR{$t=1,\dots,T$}
            \STATE Select action with noise $a \sim \pi(s) + \mathcal{N}(0, \sigma)$\;
            \STATE Observe reward $r$ and new state $s'$\;
            \STATE Store transition tuple $(s, a, r, s')$ in $\mathcal{M}$\;
            \FOR{ $i\in\{1,2\}$ }
                \STATE Sample $N$ transitions $(s_t, a_t, r_t, s'_t,R_t^{(i)})$ from $\mathcal{M}$\;
                \STATE Update $\theta_{(i)} \leftarrow \min_{\theta_{(i)}} \sum(R_t^{(i)}-Q_{\theta}^{(i)}(s_t,a_t))^2$\;
            \ENDFOR
            \IF{$t \mbox{ mod } u=0$}
                \STATE $\phi' \leftarrow \tau\phi + (1-\tau)\phi'$\;
                \STATE $\theta_{(i)}' \leftarrow \tau\theta_{(i)} + (1-\tau)\theta_{(i)}'$\;
                \STATE \textsc{Update Memory}()\;
            \ENDIF
            \IF{$t \mbox{ mod } p=0$} 
                \STATE Update $\phi$ by the deterministic policy gradient\;
                \STATE $\nabla_\phi J(\phi) = \nabla_a Q_{\theta_1}(s, a)|_{a=\pi_\phi(s)}\nabla_\phi\pi_{\phi(s)}$\;
            \ENDIF
        \ENDFOR
    \end{algorithmic}

\end{algorithm}

\begin{algorithm}[tb]
    \caption{Update Memory}
    \label{alg:updatememory}
    \begin{algorithmic}
        \FOR{trajectory $\tau$ in tabular memory $\mathcal{M}$}
        \FOR{$s_t,a_t,r_t,s_{t+1}$ in \textsc{reversed}($\tau$)}
        \STATE $\tilde{a}_{t+1} \sim \pi_{\phi'}(s_{t+1})+\text{clip}(\mathcal{N}(0,\tilde{\sigma}),-c,c)$\;
        \STATE Compute $Q_{\theta'_{(1,2)}}(s_{t+1},\tilde{a}_{t+1})$
        \STATE Compute $V_{t,h}^{(1,2)}$ with Equation~(\ref{update1}) for $h$ in $0:T-t$
        \STATE Compute $R_{t}^{(1,2)}$ with Equation~(\ref{update2}) and save into buffer $\mathcal{M}$
        \ENDFOR
        \ENDFOR
    \end{algorithmic}
\end{algorithm}
\subsection{Generalizable Episodic Memory}
Traditional discrete episodic memory is stored in a lookup table, learned as in Equation~(\ref{equ:mfec}) and used as in Equation~(\ref{equ:mfec2}). This kind of methods does not consider generalization when learning values and enjoy little generalization ability from non-parametric nearest-neighbor search with random projections during execution. To enable the generalizability of such episodic memory, we use the parametric network $\mathcal{M}_\theta$ to learn and store values. As stated in Section~\ref{overview}, this parametric memory is learned by doing regression toward the best returns $R_t$ starting from the state-action pair $(s_t,a_t)$:
\begin{equation}
    \mathcal{L}(Q_\theta)=\EX_{(s_t,a_t,R_t)\sim \mathcal{M}} (Q_\theta(s_t,a_t)-R_t)^2.
\end{equation}
where $R_t$ is computed from implicit planning over both real returns in tabular memory $\mathcal{M}$ and estimates in target generalizable memory table $\mathcal{M}_{\theta}'$, as depicted in the next section. This learned virtual memory table is then used for both policy learning and building new target.

\subsection{Implicit Memory-Based Planning}
To leverage the analogical reasoning ability of our parametric memory, GEM conducts implicit memory-based planning to estimate the value for the best possible rollout for each state-action pair. At each step, GEM compares the best return so far along the trajectory with the value estimates from $Q_\theta$ and takes the maximum between them. $Q_\theta$  is generalized from similar experiences and can be regraded as the value estimates for counterfactual trajectories. This procedure is conducted recursively from the last step to the first step along the trajectory, forming an implicit planning scheme within episodic memory to aggregate experiences along and across trajectories. The overall back-propagation process can be written in the following form:
\begin{equation}
    \label{singlebp}
    R_t=\left\{
    \begin{aligned}
        &r_t+\gamma \max(R_{t+1}, Q_\theta(s_{t+1},a_{t+1})) &\mbox{if } t<T, \\
        &r_t & \mbox{if } t=T, \\
    \end{aligned}
    \right.
\end{equation}
where $t$ denotes steps along the trajectory and $T$ the episode length.
Further, the back-propagation process in Equation~(\ref{singlebp}) can be unrolled and rewritten as follows:

\begin{align}
    \label{update0}
    V_{t,h}&=
    \left\{
        \begin{aligned}
            &r_t+\gamma V_{t+1,h-1} & \mbox{  if } h>0, \\
            &Q_\theta(s_t,a_t) & \mbox{  if } h=0, \\
        \end{aligned}
        \right. \notag\\
    R_t&=V_{t,h^*},h^*=\argmax_{h>0} V_{t,h}, 
\end{align}
where $h$ denotes different length of rollout steps, and we define $Q_\theta(s_t,a_t)=0$ and $V_{t,h}=0$ for $t>T$.\par
\subsection{Twin Back-Propagation Process}
\label{TBN}
In this section, we describe our novel twin back-propagation process to reduce trajectory-induced overestimation.\par
As mentioned in Section~\ref{overview}, using Equation~(\ref{update0}) for planning directly can lead to severe overestimation, even if $Q_\theta$ is unbiased. Formally, for any unbiased set of estimators $\tilde{Q}_{h}(s,a)=Q_{h}(s,a)+U_h(s,a)$, where $Q_{h}(s,a)$ is the true value and $U_h(s,a)$ is a set of independent, zero-mean random noise, we have
\begin{equation}
    \EX_U\left[\max_h \tilde{Q}_{h}(s,a)\right] \geq \max_h Q_{h}(s,a),
\end{equation}
which can be derived directly from Jensen's inequality. \par
Thus trajectory-augmented values have a tendency to overestimate and this overestimation bias can hurt performance greatly. We propose to use twin networks, $Q^{(1)}, Q^{(2)}$, to form a double estimator for estimating the best value $R_t$ in the value back-propagation in Equation~(\ref{update0}). One network is used to estimate the length of rollouts with the maximum returns along a trajectory ($h^*$ in Equation~(\ref{update0})), while the other estimates how large the return is following this estimated best rollout ($V_{t,h^*}$ in Equation~(\ref{update0})). Formally, the proposed twin back-propagation process (TBP) is given by \footnote{In Equation~(\ref{update1}) \& (\ref{update2}),  index  $(1,2)$ represents it can be either $(1)$ or $(2)$, while reversed index $(2,1)$ refers to the other network's estimation, i.e. represents $(2)$ when $(1,2)$ refers to $(1)$ and vice versa.}:
\begin{align}
    \label{update1}
    V^{(1,2)}_{t,h}&=
    \left\{
        \begin{aligned}
            r_t+\gamma V^{(1,2)}_{t+1,h-1} & \mbox{  if } h>0, \\
            Q^{(1,2)}_\theta(s_t,a_t) & \mbox{  if } h=0, \\
        \end{aligned}
     \right. \\
    \label{update2}
    R_t^{(1,2)}&=V_{t,h^*_{(1,2)}}^{(2,1)},h^*_{(1,2)}=\argmax_{h>0} V^{(1,2)}_{t,h}.
\end{align}
And each $Q$-network is updated by
\begin{align}
    \mathcal{L}(\theta_{(1,2)})&=\EX_{s_t,a_t,R_t\sim\mathcal{M}}\left(Q_{\theta}^{(1,2)}(s_t,a_t)-R^{(1,2)}_t\right)^2. \notag
\end{align}

Note that this twin mechanism is different from double Q-learning, where we take the maximum over timesteps while double Q-learning takes the maximum over actions.\par
Formal analysis in Section~\ref{nonoverestimation} shows that this twin back-propagation process does not overestimate the true maximum expectation value, given unbiased $Q_\theta$. As stated in \citet{DoubleQ} and \citet{TD3}, while this update rule may introduce an underestimation bias, it is much better than overestimation, which can be propagated along trajectories. 
\subsection{Conservative Estimation on Single Step}
\label{conservative}
The objective in Equation~(\ref{update2}) may still prone to overestimation since $Q^{(1,2)}_\theta$ is not necessarily unbiased, and it is well known that deterministic policy gradient has a tendency to overestimate \citep{TD3}. To address this issue, we further augment each $Q^{(1,2)}_\theta$ functions with two independent networks $Q_A,Q_B$, and use clipped double Q-learning objective as in \citet{TD3} for the estimation at each step:
\begin{equation}
    Q^{(1,2)}(s,a) = \min_{A,B}\left(Q^{(1,2)}_A(s,a), Q^{(1,2)}_B(s,a)\right).   
\end{equation}

Finally, we propose the following asymmetric objective to penalize overestimated approximation error,
\begin{equation}
    \mathcal{L}(\theta_{(1,2)}) = \EX_{s_t,a_t,R_t\sim\mathcal{M}}[ (\delta_t)_+^2+\alpha(-\delta_t)_+^2],
\end{equation}
where $\delta_t=Q^{(1,2)}(s_t,a_t)-R^{(1,2)}_t$ is the regression error and $(\cdot)_+=\max(\cdot,0)$. $\alpha$ is the hyperparameter to control the degree of asymmetry.\par

\section{Theoretical Analysis}
In this section, we aim to establish a theoretical characterization of our proposed approach through several aspects. We begin by showing an important property that the twin back-propagation process itself does not have an incentive to overestimate values, the same property as guaranteed by double Q-learning. In addition, we prove that our approach guarantees to converge to the optimal solution in deterministic environments. Moreover, in stochastic environments, the performance could be bounded by a dependency term regrading the environment stochasticity.

\subsection{Non-Overestimation Property}
\label{nonoverestimation}

We first investigate the algorithmic property of the twin back-propagation process in terms of \textit{value estimation bias}, which is an essential concept for value-based methods \citep{thrun1993issues, DoubleQ, TD3}. Theorem \ref{theorem1} indicates that our method would not overestimate the real maximum value in expectation.

\begin{theorem}
    \label{theorem1}
    Given unbiased and independent estimators $\tilde{Q}^{(1,2)}_\theta(s_{t+h},a_{t+h})=Q^\pi(s_{t+h},a_{t+h})+\epsilon^{(1,2)}_h$, Equation~(\ref{update2}) will not overestimate the true objective, i.e.
    \begin{equation}
        \EX_{\tau,\epsilon} \left[R_t^{(1,2)}(s_t)\right]\leq  \EX_{\tau}\left[\max_{0\leq h < T-t}Q^\pi_{t,h}(s_t,a_t)\right],
    \end{equation}
    where $Q^\pi_{t,h}(s,a)=$
    \begin{equation}
        \sum_{i=0}^{h}\gamma^i r_{t+i}+
        \left\{\begin{aligned}
        &\gamma^{h+1} Q^\pi(s_{t+h+1},a_{t+h+1}),& \text{if } h<T-t,\\
        &0,& \text{if } h=T-t.
    \end{aligned}
    \right.
    \end{equation}
    and $\tau=\{(s_t,a_t,r_t,s_{t+1})_{t=1,\cdots,T}\}$ is a trajectory.
\end{theorem}

The proof of Theorem \ref{theorem1} is deferred to Appendix~\ref{proofs}. As revealed by Theorem \ref{theorem1}, the twin back-propagation process maintains the same non-overestimation nature of double Q-learning \citep{DoubleQ}, which ensures the reliability of our proposed value propagation mechanism.

\subsection{Convergence Property}

In addition to the statistical property of value estimation, we also analyze the convergence behavior of GEM. Following the same environment assumptions as related work \citep{MFEC, associative}, We first derive the convergence guarantee of GEM in deterministic scenarios as the following statement.

\begin{theorem}
    \leavevmode
    \label{theorem2}
    In a finite MDP with a discount factor $\gamma<1$, the tabular parameterization of Algorithm \ref{alg:amc} would converge to $Q^*$ w.p.1 under the following conditions:
    \begin{enumerate}
        \item $\sum_t{\alpha_t(s,a)}=\infty,\sum_t{\alpha_t^2(s,a)}<\infty$
        \item The transition function of the given environment is fully deterministic, i.e., $P(s'|s,a)=\delta(s'=f(s,a))$ for some deterministic transition function $f$
    \end{enumerate}
    where $\alpha_t\in(0,1)$ denotes the scheduling of learning rates. A formal description for tabular parameterization of Algorithm \ref{alg:amc} is included in Appendix~\ref{tabular}.
\end{theorem}
The proof of Theorem \ref{theorem2} is extended based on \citet{DoubleQ}, and details are deferred to Appendix~\ref{proofs}.
Note that this theorem only applies to deterministic scenarios, which is a common assumption for memory-based algorithms \citep{MFEC, associative}. To establish a more precise characterization, we consider a more general class of MDPs, named near-deterministic MDPs, as stated in Definition \ref{def:near-deterministic}.
\begin{definition}\label{def:near-deterministic}
    We define $Q_{\text{max}}(s_0,a_0)$ as the maximum value possible to receive starting from $(s_0,a_0)$, i.e.,
    \begin{equation*}
        Q_{\text{max}}(s_0,a_0)\coloneqq \max_{
            \substack{ (s_1,\cdots,s_T),(a_1,\cdots,a_T)\\
            s_{i+1}\in supp(P(\cdot|s_i,a_i))}
    }{\sum_{t=0}^T \gamma^t r(s_t,a_t)}.
    \end{equation*}
    An MDP is said to be nearly-deterministic with parameter $\mu$, if $\forall s\in \mathcal{S}, a \in \mathcal{A},$
    \begin{align*}
        \qquad\qquad Q_{\text{max}}(s,a)\leq Q^*(s,a)+\mu,
    \end{align*}
    where $\mu$ is a dependency threshold to bound the stochasticity of environments.
\end{definition}

Based on the definition of near-deterministic MDPs, we formalize the performance guarantee of our approach as the following statements:
\begin{lemma}
    \label{theorem3}
    The value function $Q(s,a)$ learned by the tabular variant of Algorithm \ref{alg:amc} satisfies the following inequality:
    \begin{equation*}
        \forall s\in \mathcal{S},a\in\mathcal{A},Q^*(s,a)\leq Q(s,a) \leq Q_{\text{max}}(s,a),
    \end{equation*}
     w.p.1, under condition 1 in Theorem \ref{theorem2}.
\end{lemma}
\begin{proof}[Proof Sketch]
    We only need to show that $\norm{(Q-Q^*)_+}_\infty$ is a $\gamma$-contraction and also $\norm{(Q_{\text{max}}-Q)_+}_\infty$. The rest of the proof is similar to Theorem \ref{theorem2}.
\end{proof}

\begin{theorem} \label{thm-near-deterministic}
    For a nearly-deterministic environment with factor $\mu$, GEM's performance can be bounded w.p.1 by 
    $$V^{\tilde{\pi}}(s)\geq V^*(s)-\frac{2\mu}{1-\gamma},\forall s \in \mathcal{S}.$$
\end{theorem}
The complete proof of these statements are deferred to Appendix~\ref{proofs}. Theorem \ref{thm-near-deterministic} ensures that, our approach is applicable to near-deterministic environments as most real-world scenarios.

\section{Experiments}
Our experimental evaluation aims to answer the following questions: (1) How well does GEM perform on the continuous state and action space? (2) How well does GEM perform on discrete domains? (3) How effective is each part of GEM?

\begin{figure*}[htb]
    \centering
    \begin{subfigure}[c]{0.33\textwidth}
     \centering
     \includegraphics[width=\linewidth]{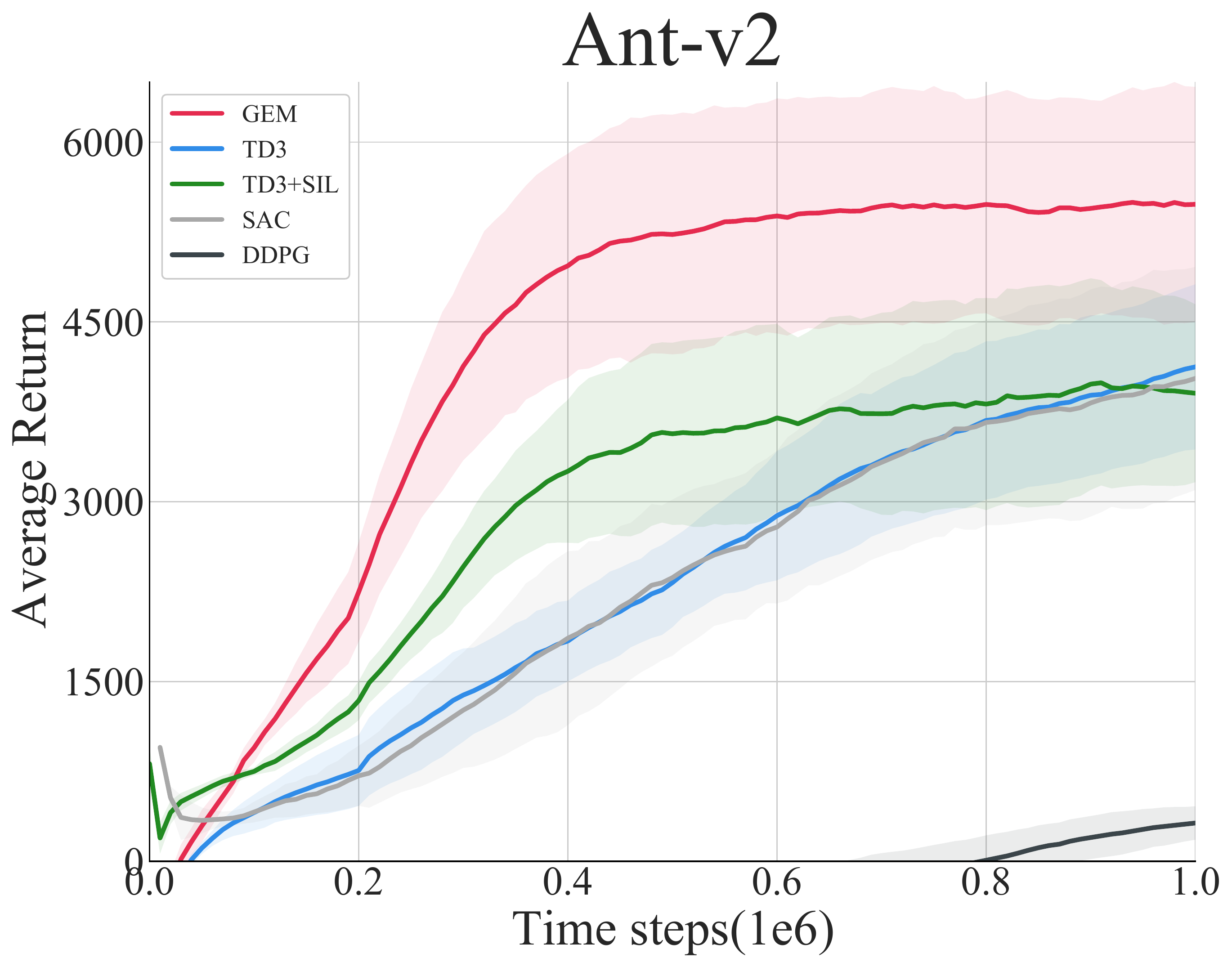}
    \end{subfigure}
    \begin{subfigure}[c]{0.33\textwidth}
     \centering
     \includegraphics[width=\linewidth]{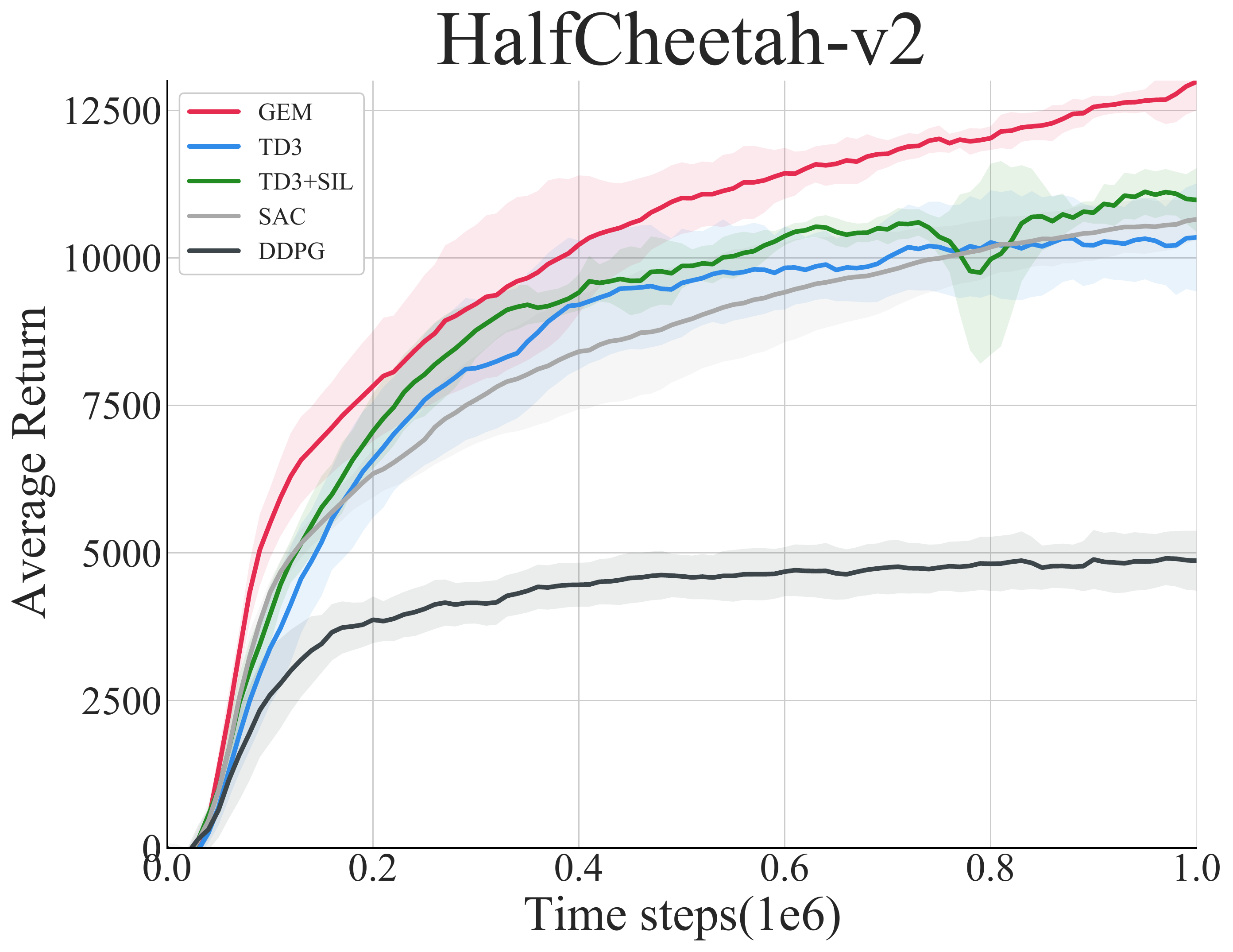}
    \end{subfigure}
    \begin{subfigure}[c]{0.33\textwidth}
     \centering
     \includegraphics[width=\linewidth]{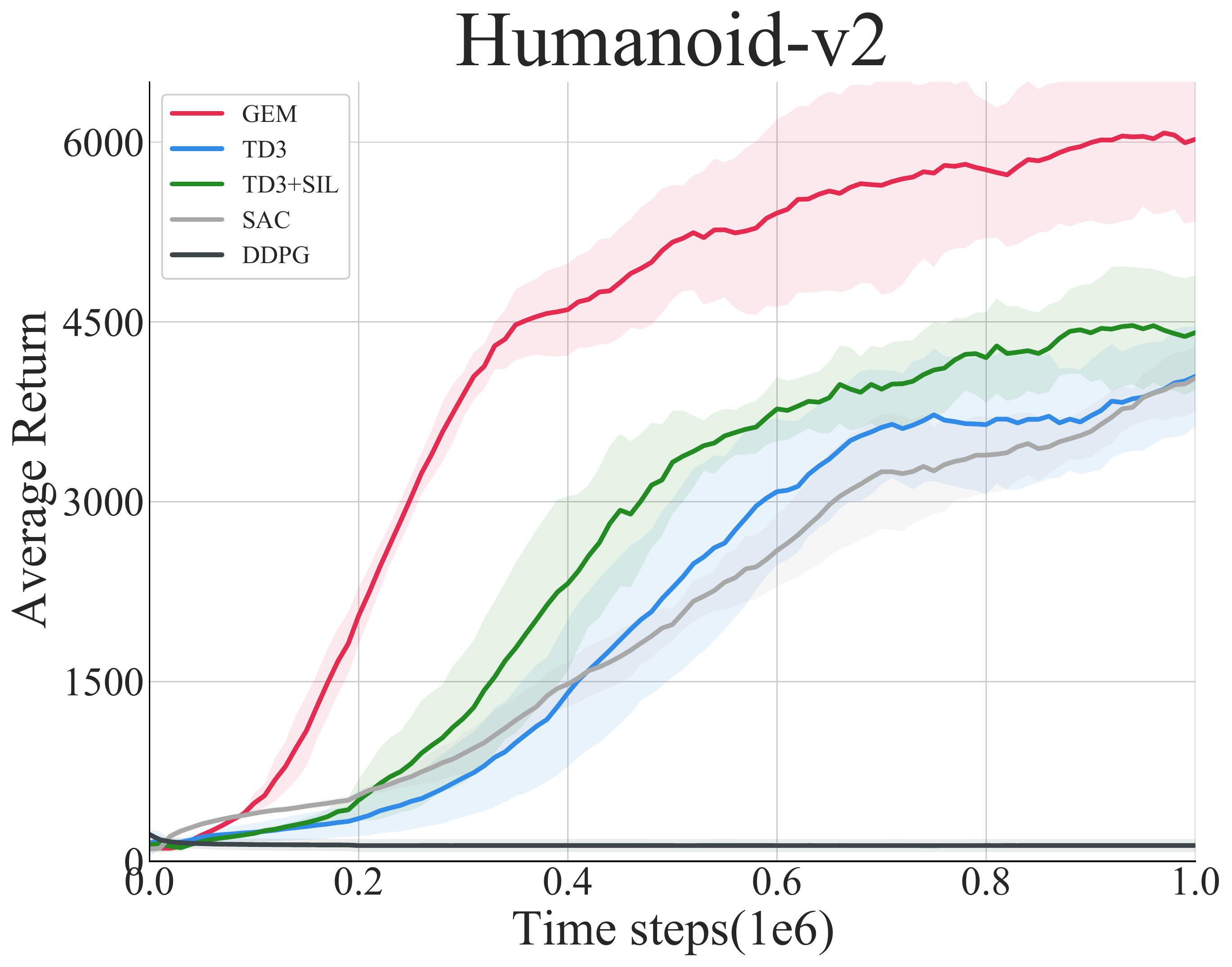}
    \end{subfigure}
    \begin{subfigure}[c]{0.33\textwidth}
     \centering
     \includegraphics[width=\linewidth]{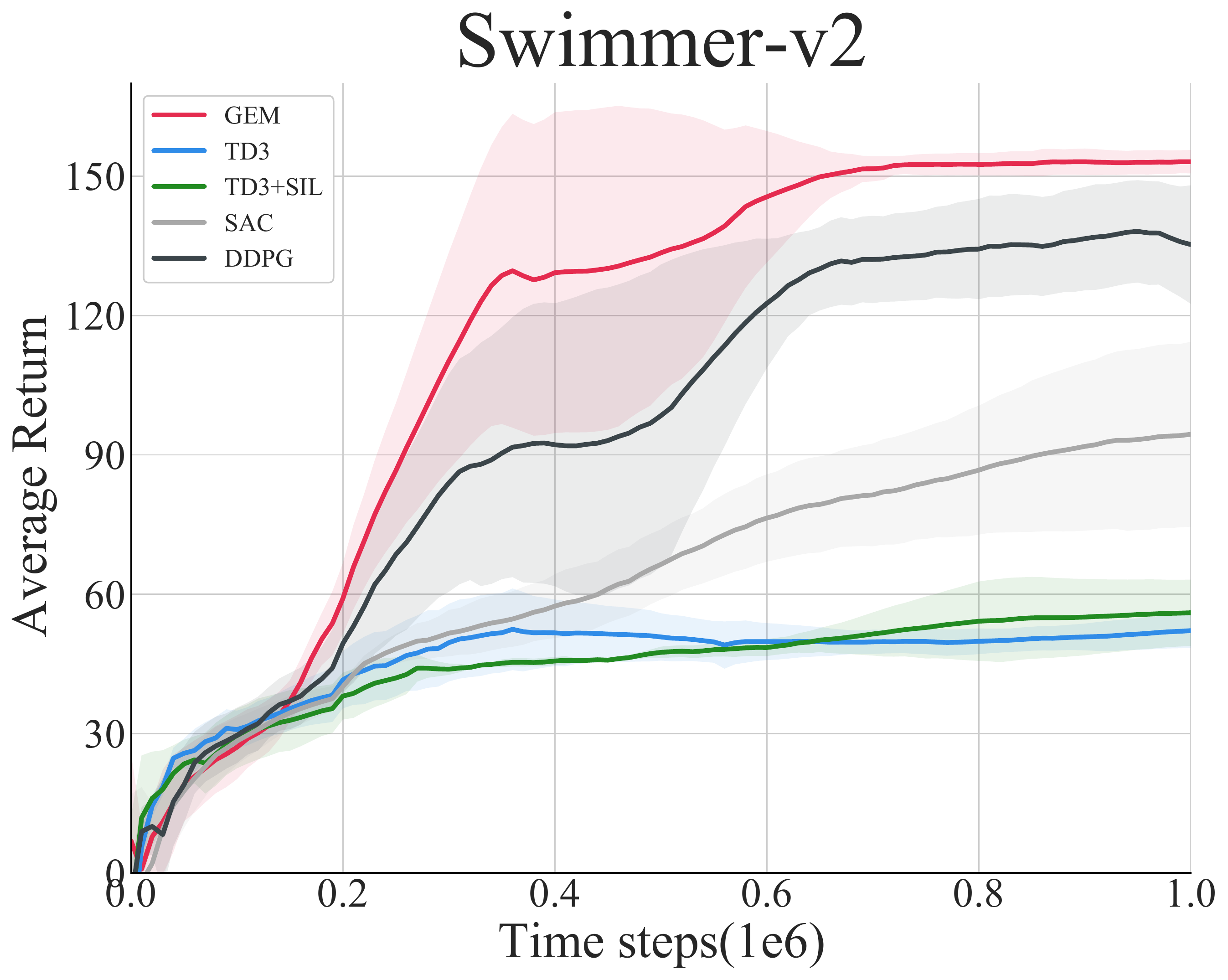}
    \end{subfigure}
    \begin{subfigure}[c]{0.33\textwidth}
     \centering
     \includegraphics[width=\linewidth]{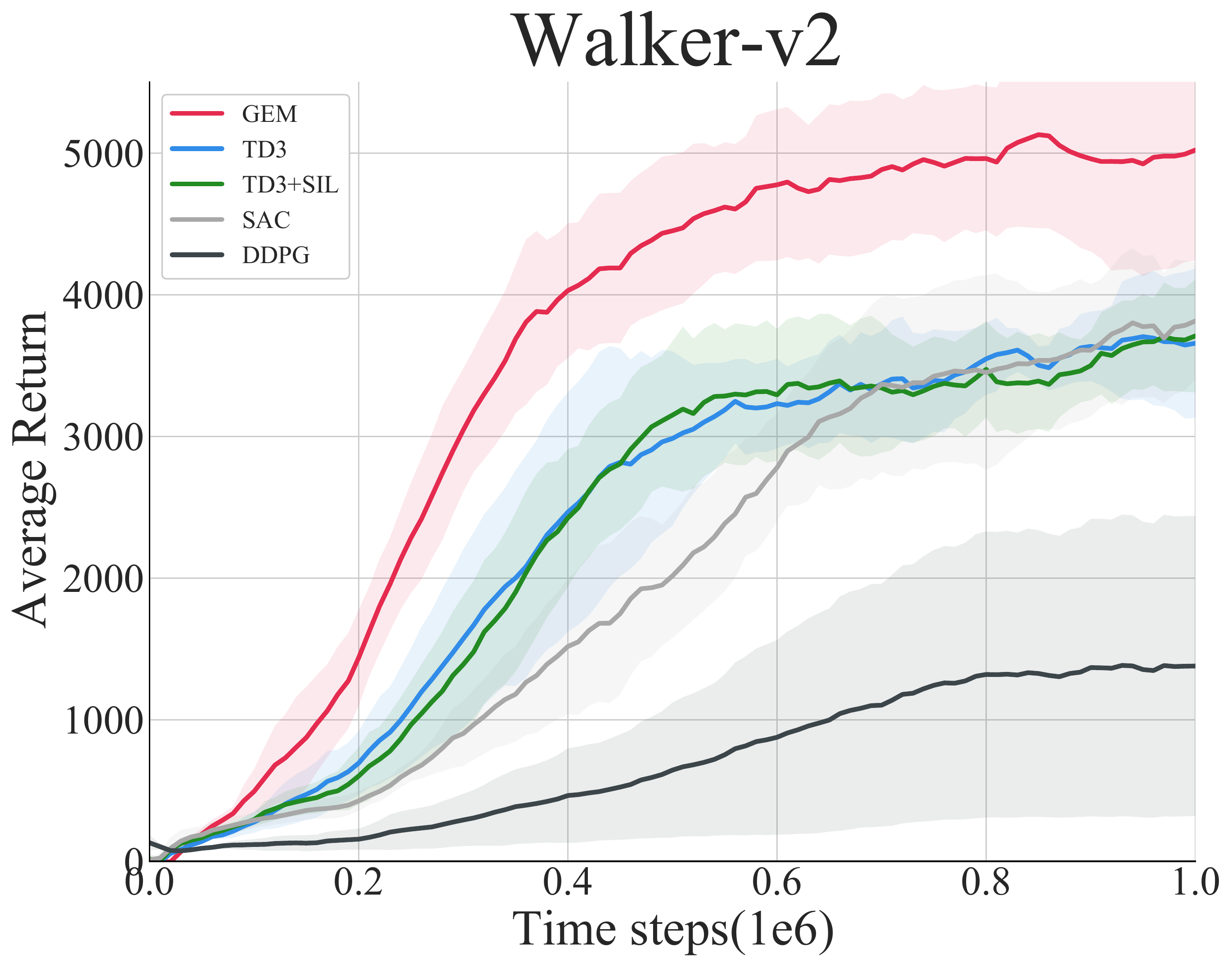}
    \end{subfigure}
    \begin{subfigure}[c]{0.33\textwidth}
     \centering
     \includegraphics[width=\linewidth]{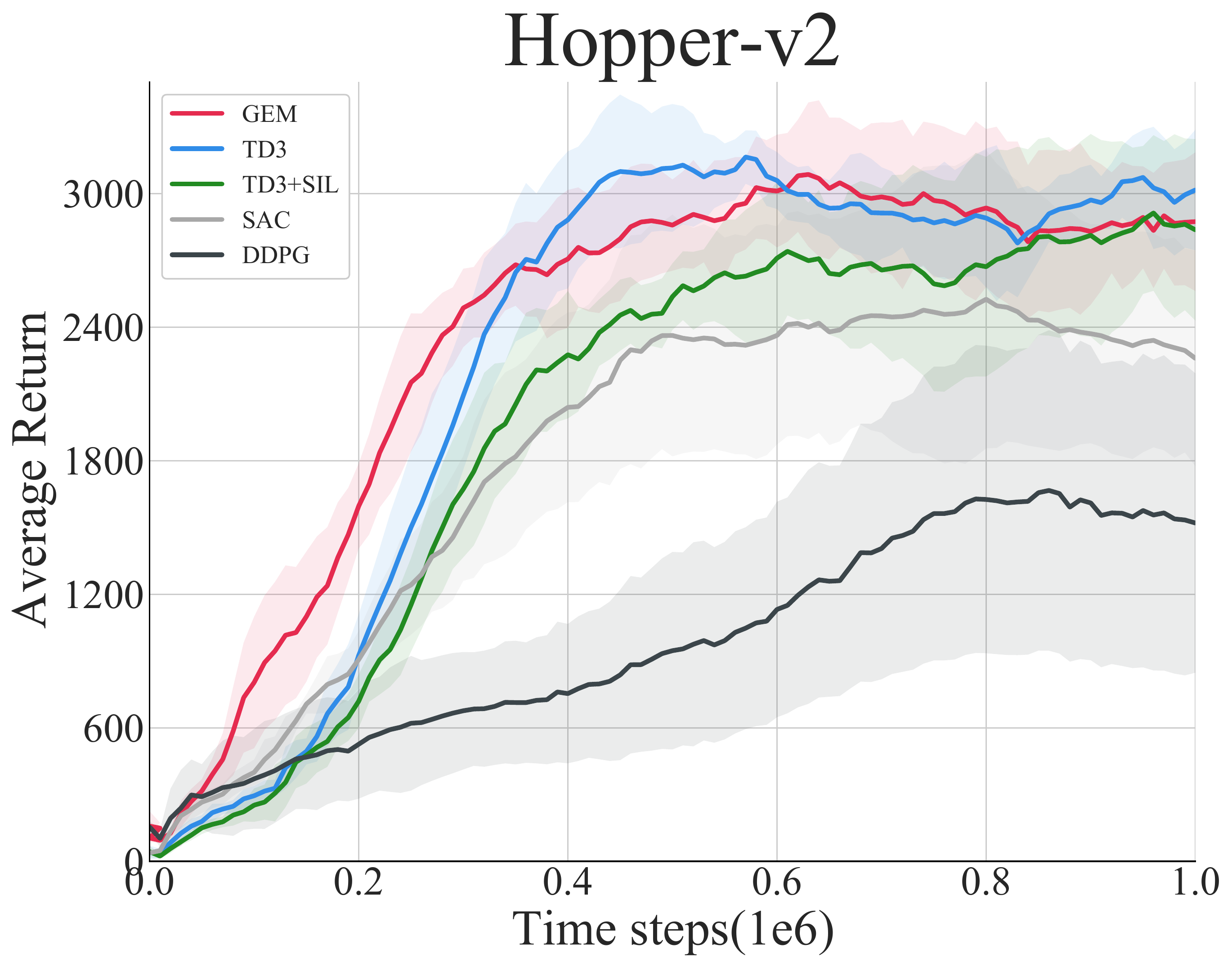}
    \end{subfigure}
    \caption{Learning curves on MuJoCo tasks, compared with baseline algorithms. The shaded region represents the standard deviation of the performance. Each curve is averaged over five seeds and is smoothed for visual clarity.}
    \label{performance}
   \end{figure*}

\subsection{Evaluation on Continuous Control Tasks}
We compare our method with the most popular model-free RL algorithms, including DDPG \citep{DDPG}, TD3 \citep{TD3} and SAC \citep{SAC}. Conventional episodic RL such as MFEC \citep{MFEC}, NEC \citep{NEC}, EMDQN \citep{EMDQN}, EVA \citep{EVA} and ERLAM \citep{associative} adopts discrete episodic memory that are only designed for discrete action space, thus we cannot directly compare with them. To compare with episodic memory-based approaches on the continuous domain, we instead choose self-imitation learning \citep[SIL;][]{SIL} in our experiments. It also aims to exploit past good experiences and share a similar idea with episodic control; thus, it can be regarded as a continuous version of episodic control. For a fair comparison, We use the same code base for all algorithms, and we combine the self-limitation learning objective with techniques in TD3 \citep{TD3}, rather than use the origin implementation, which combines with PPO \citep{PPO}.\par

We conduct experiments on the suite of MuJoCo tasks \citep{mujoco}, with OpenAI Gym interface \citep{brockman2016openai}. We truncate the maximum steps available for planning to reduce the overestimation bias. The memory update frequency $u$ is set to 100 with a smoothing coefficient $\tau=0.6$. The rest of the hyperparameters are mostly kept the same as in TD3 to ensure a fair comparison. The detailed hyperparameters used are listed in Appendix~\ref{hyperparameters}. \\

The learning curve on different continuous tasks is shown in Figure \ref{performance}. We report the performance of 1M steps, which is evaluated with $10$ rollouts for every $10000$ steps with deterministic policies. As the results suggest, our method significantly outperforms other baseline algorithms on most tasks. Only on Hopper, GEM is not the absolute best, but all the algorithms have similar performance. 

\subsection{Evaluation on Discrete Domains}
Though GEM is proposed to facilitate continuous control tasks, it also has the general applicability of boosting learning on discrete domains. To demonstrate this, we compare GEM with several advanced deep Q-learning algorithms for discrete action space, including DQN \citep{DQN}, DDQN \citep{van2016deep}, and Dueling DQN \citep{wang2016dueling}. We also include the clipped double DQN, which is adopted from a state-of-the-art algorithm of the continuous domain, TD3 \citep{TD3}. Most of these baseline algorithms focus on learning stability and the reduction of estimation bias. We evaluate all the above algorithms on 6 Atari games \citep{Atari}. We use hyper-parameters suggested in Rainbow \citep{rainbow}, and other hyper-parameters for our algorithm are kept the same as in the continuous domain. As shown in Figure~\ref{atari}, although not tailored for the discrete domain, GEM significantly outperforms baseline algorithms both in terms of sample efficiency and final performance.
\begin{figure*}[htb]
    \centering
    \includegraphics[width=15cm]{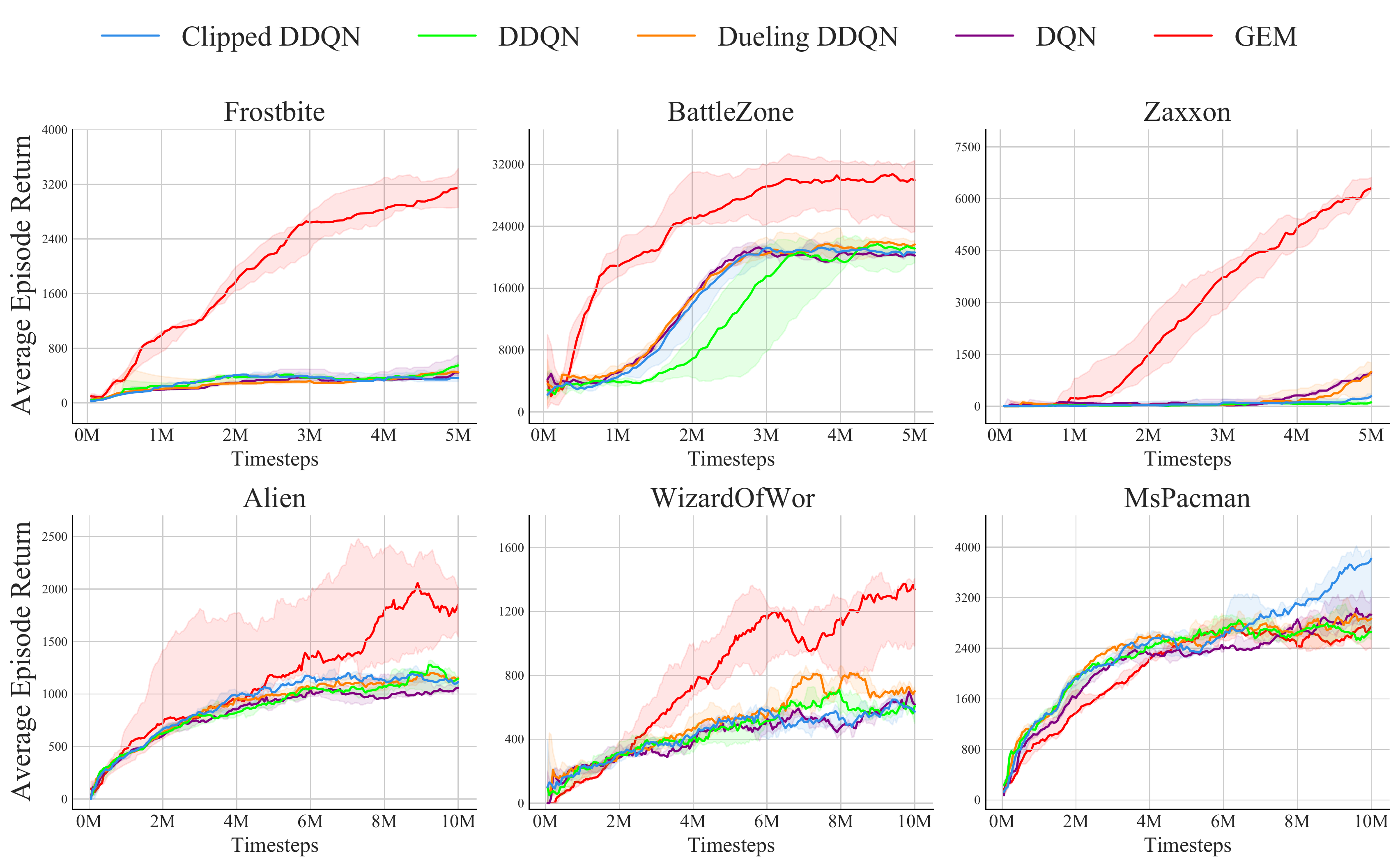}

    \caption{Performance comparison on six Atari games supported by OpenAI gym. Each curve is averaged over three seeds. The shaded region represents the standard deviation of the performance.}
    \label{atari}
\end{figure*}

\subsection{Ablation Study}
This section aims to understand each part's contribution to our proposed algorithm, including the generalizable episodic memory, implicit memory-based planning, and twin back-propagation. We also empirically demonstrate the effect of overestimation. To check whether our method, which uses four networks,  benifits from ensembling effect, we also compare our method with the random ensembling method as used in REDQ \citep{REDQ}.\par

To verify the effectiveness of generalizable episodic memory, we compare our algorithm with SIL, which directly uses historical returns (which can be seen as a discrete episodic memory). As shown in both Figure \ref{performance} and Figure \ref{overestimation}, although SIL improves over TD3, it only has marginal improvement. On the contrary, our algorithm improves significantly over TD3.\par 

\begin{figure*}[htb]
    \centering
    \begin{subfigure}[c]{0.33\textwidth} 
        \centering
        \includegraphics[width=\linewidth]{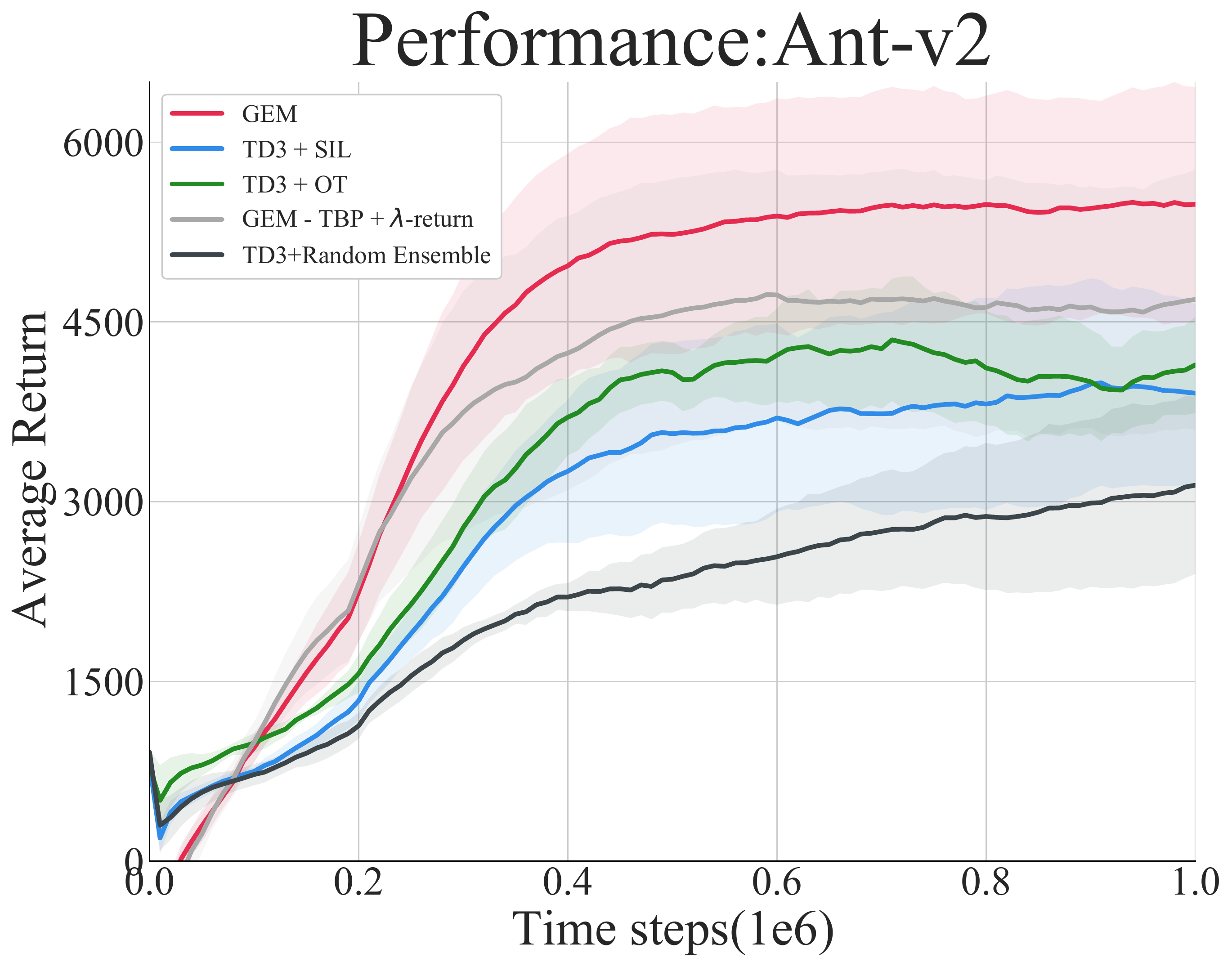}
    \end{subfigure}
    \begin{subfigure}[c]{0.33\textwidth}
        \centering
        \includegraphics[width=\linewidth]{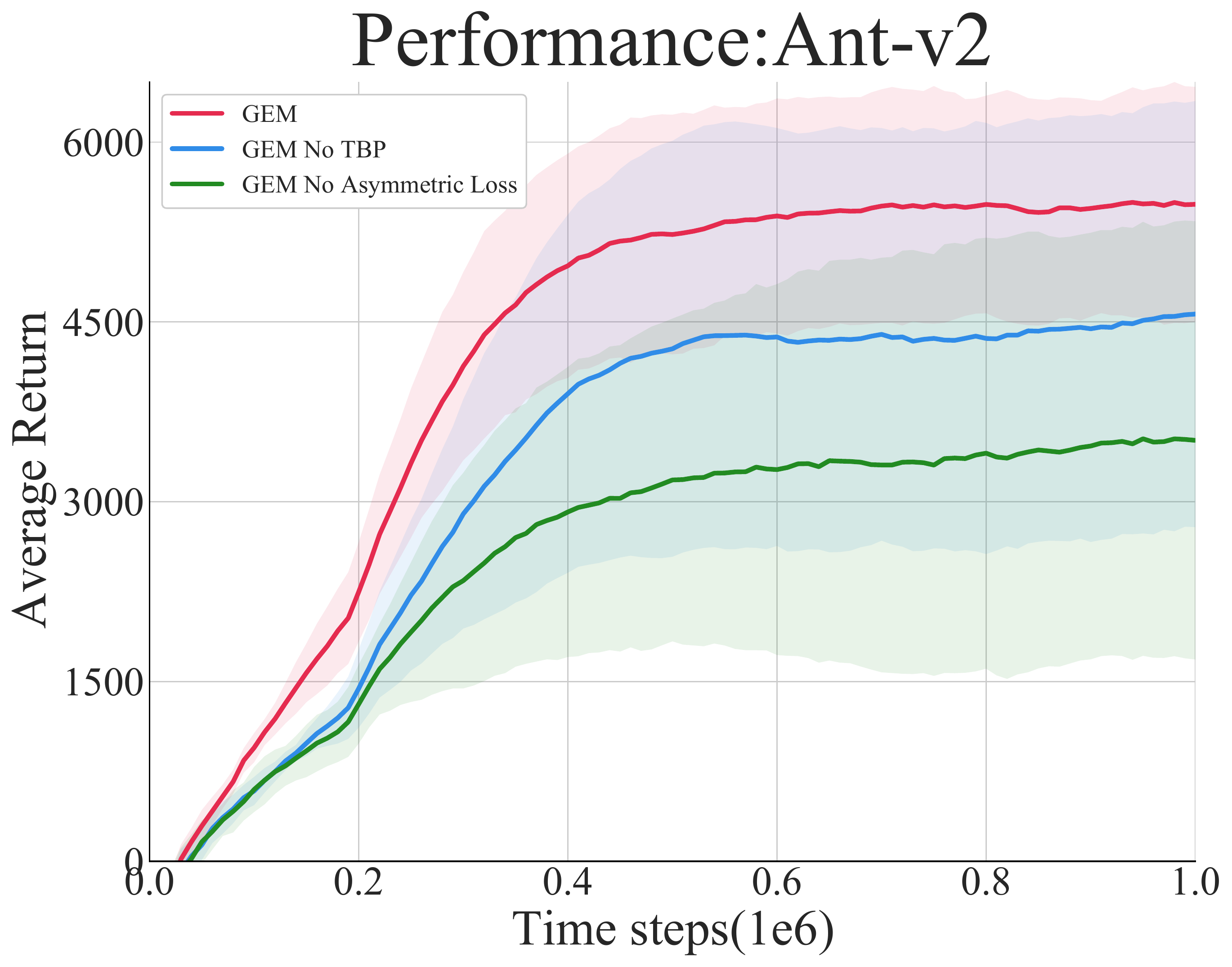}
    \end{subfigure}
    \begin{subfigure}[c]{0.33\textwidth} 
        \centering
        \includegraphics[width=\linewidth]{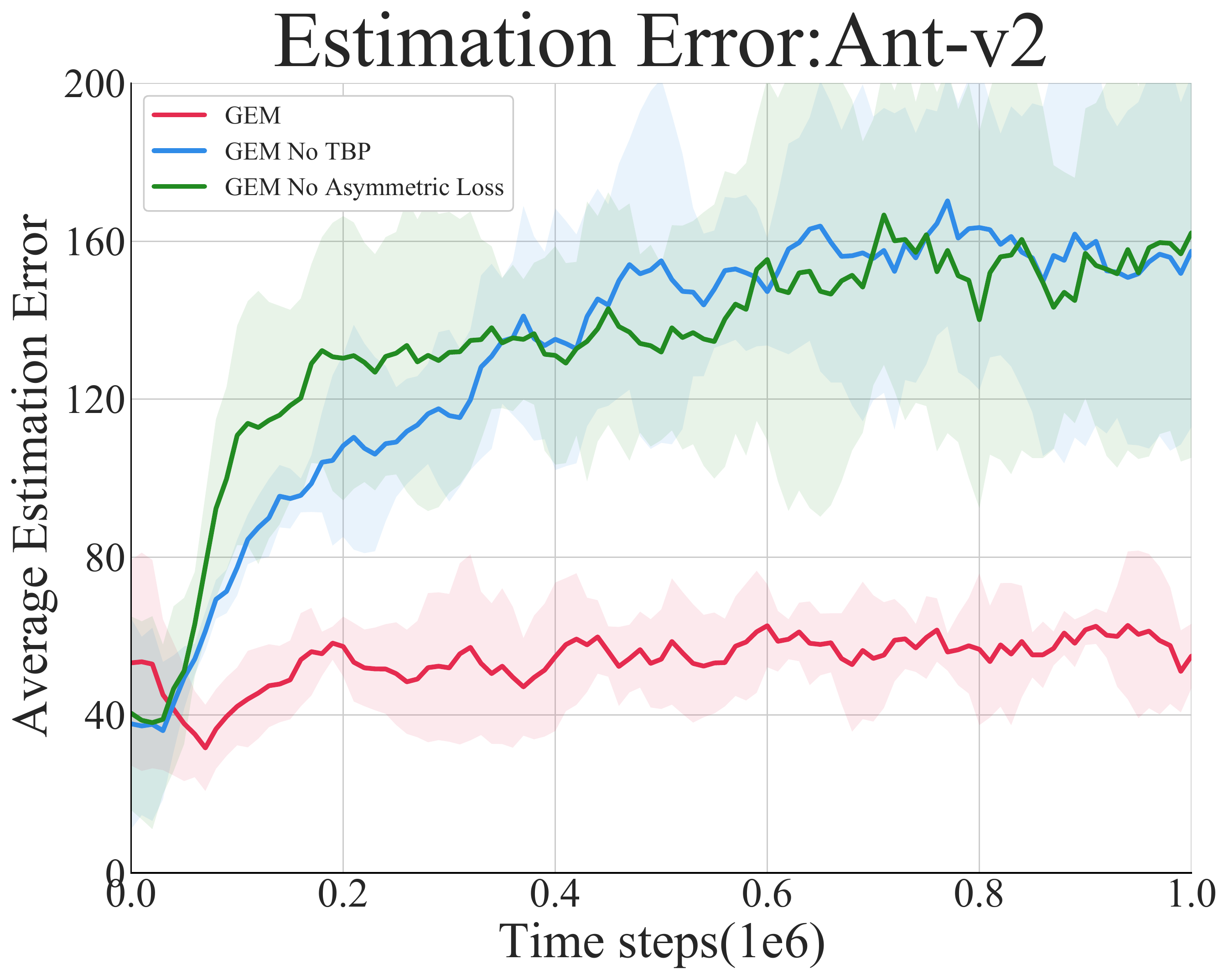}
    \end{subfigure}
    \newline
    \begin{subfigure}[c]{0.33\textwidth} 
        \centering
        \includegraphics[width=\linewidth]{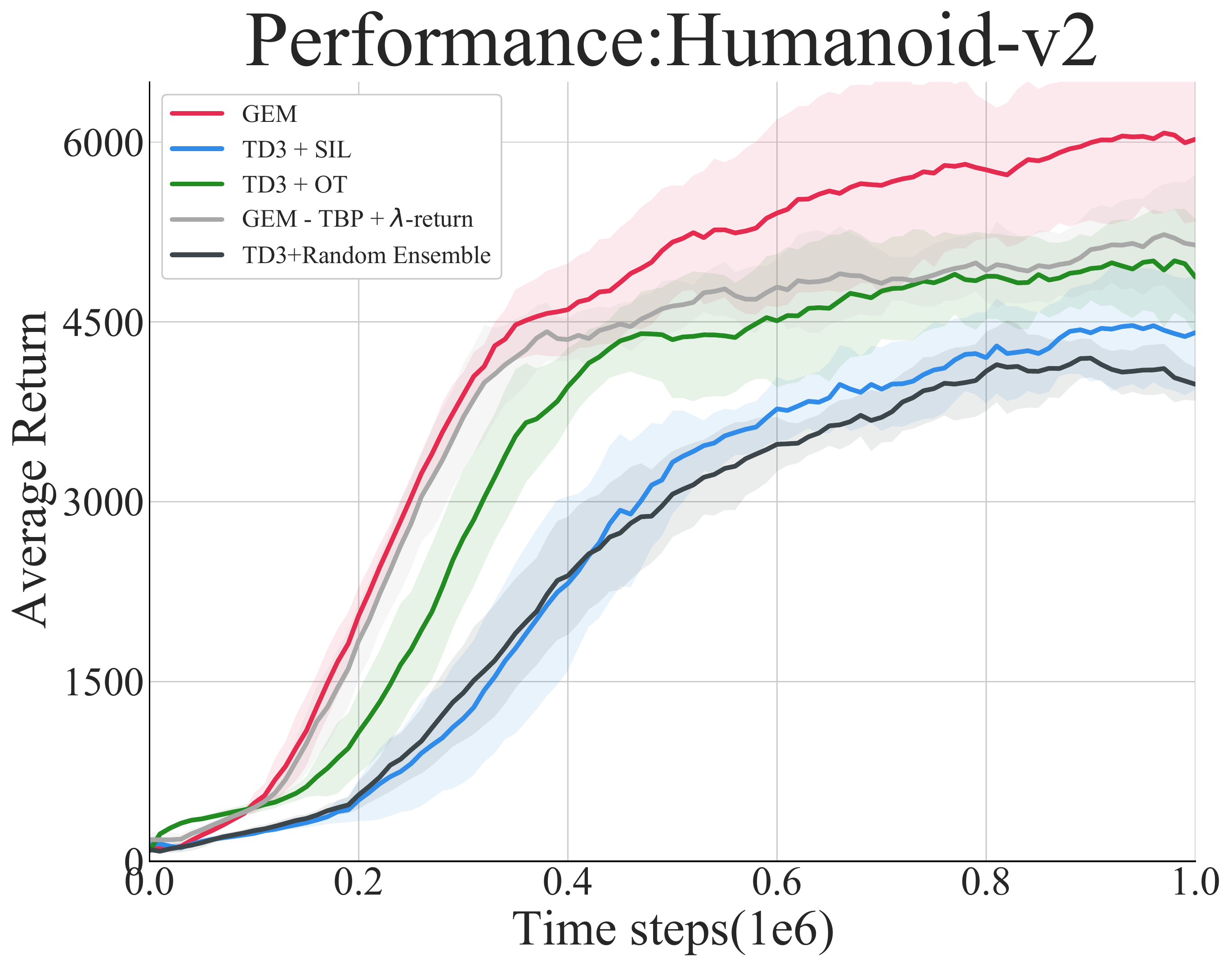}
    \end{subfigure}
    \begin{subfigure}[c]{0.33\textwidth}
        \centering
        \includegraphics[width=\linewidth]{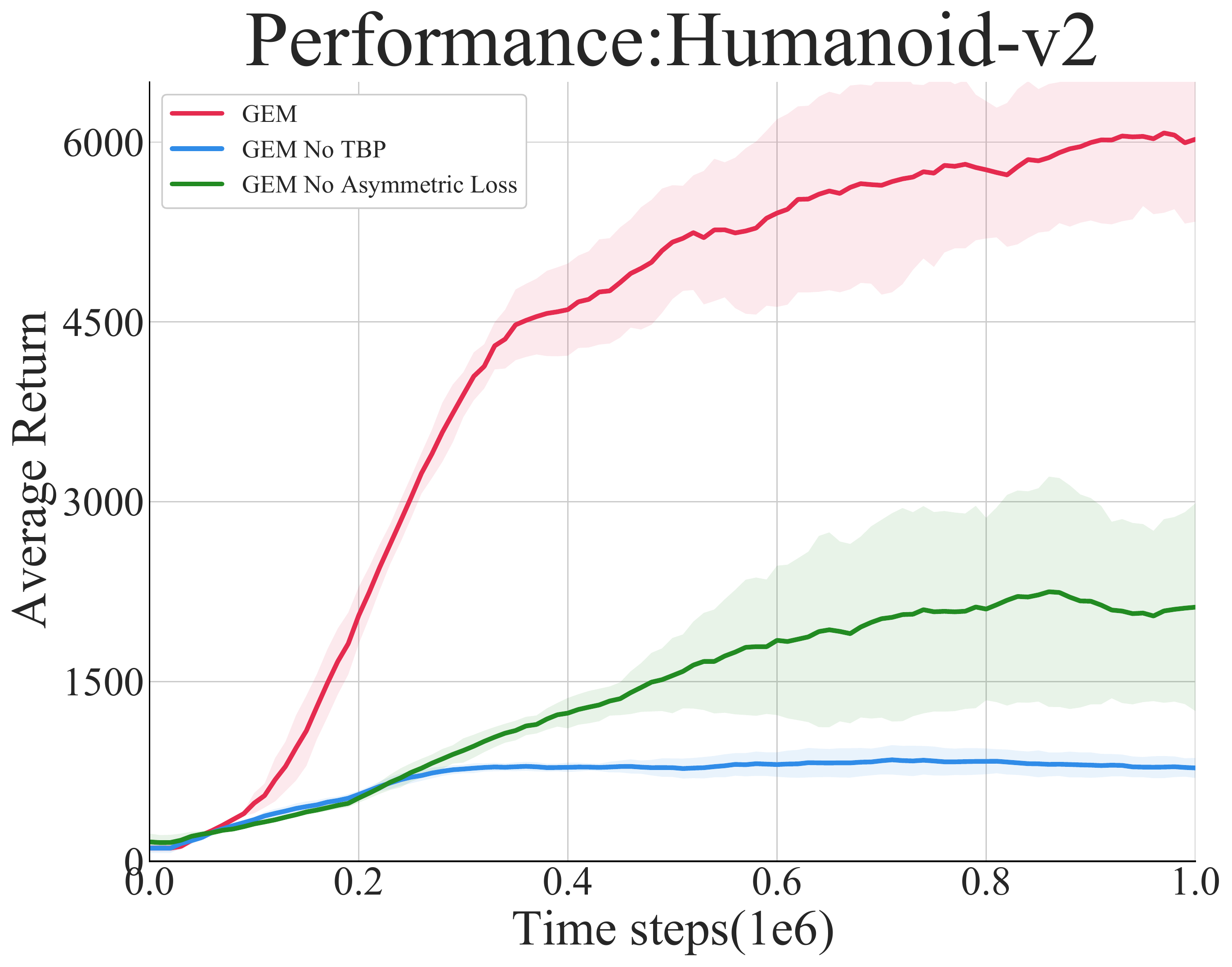}
    \end{subfigure}
    \begin{subfigure}[c]{0.33\textwidth}
    \centering
    \includegraphics[width=\linewidth]{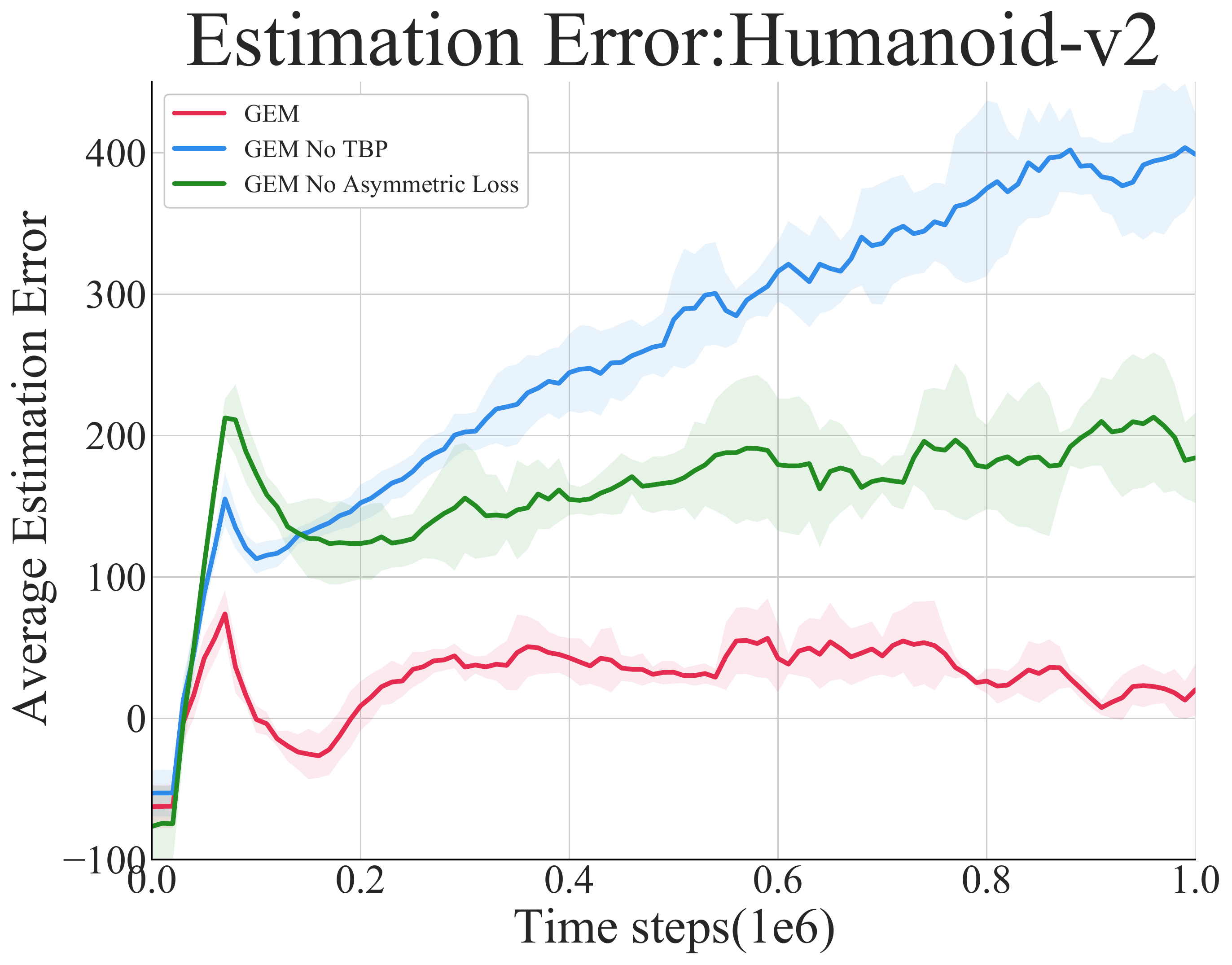}
    \end{subfigure}
    
    \caption{Addtional comparison, ablation study of GEM and empirical evaluation for overestimation. The shaded region represents the standard deviation of the performance. Each curve is averaged over five seeds and is smoothed for visual clarity. Estimation error refers to the average estimated Q-values minus the average returns.}
    \label{overestimation}
\end{figure*}

To understand the contribution of implicit planning, we compare with a variant of our method that uses $\lambda$-return rather than taking the maximum over different steps. We also compare our method with Optimality Tightening \citep{he2016learning}, which shares the similar idea of using trajectory-based information. As Figure~\ref{overestimation} shows, Our method outperforms both methods, which verifies the effectiveness of implicit planning within memories.\par

To check the contribution of our twin network and conservative estimation, we compare the results without the twin back-propagation process and the results without asymmetric losses. The result and the estimation error are summarized in Figure~\ref{overestimation}. We can see that overestimation is severe without TBP or conservative estimation, and the performance is greatly affected, especially in games like Humanoid. In these games, the living condition for the agent is strict; thus overestimation is more severely punished.\par 

To understand the contribution of the ensembling effect in our method, we compare our result with random ensembling as used in \citet{REDQ}, since we are not able to directly remove TBP while still using all four networks. We use the same update-to-data ratio as in GEM for a fair comparison. As Figure~\ref{overestimation} shows, naive ensembling contributes little to performance, and overestimation reduction contributes much more than the ensembling effect. \\

From these comparisons, we can conclude that our generalizable memory aggregates return values much more effectively than directly using historical returns, and the twin back-propagation process contributes to the performance of GEM significantly. In addition, we verify that the merits of the twin back-propagation process mainly come from overestimation reduction rather than ensembling.

\section{Related Work}
\paragraph{Episodic Control.} 
Our method is closely related to episodic reinforcement learning, pioneered by \citet{MFEC}, which introduces the concept of episodic memory into reinforcement learning and uses a non-parametric k-NN search to recall past successful experiences quickly. \citet{NEC} and \citet{EMDQN} consider to include a parametric counterpart for episodic memory, which enables better generalizability through function approximators.  Although they provide some level of generalization ability, the aggregation of different trajectories still relies on the exact re-encountering, and their methods cannot handle continuous action space. Several advanced variants \citep{associative, EBU} propose several ways to connect different experiences within the episodic memory to improve learning speed further. However, their methods are still not generally applicable in the continuous domain since they require taking maximum among all possible actions or require exact re-encountering. \citet{EVA} adopts the similar idea of soft aggregation as ours, but they only consider decision time and leaves the ability of learning by analogy aside. Self-imitation learning \citep{SIL}, another memory-based approach built upon policy gradient methods, uses a return-based lower bound to restrict the value function, which cannot aggregate associative trajectories effectively.

\paragraph{From Discrete to Continuous.}
Researchers have long been attracted to the challenge of making the RL algorithm applicable in the general continuous case. Kernel-based reinforcement learning \citep{ormoneit2002kernel} proposes a way to overcomes the stability problems of temporal-difference learning in continuous state-spaces. \citet{DDPG} extends the success of deep Q-learning into the continuous domain, which successes on various MuJoCo tasks. Similarly, in episodic RL, GEM generalizes the idea of a discrete episodic memory by representing the episodic memory as a network-generated virtual table, making it generally applicable in the continuous case.
\paragraph{Maximization Bias in Q-Learning.}
The maximization bias of Q-learning, first highlighted by \citet{thrun1993issues}, is a long-lasting issue that hinders the learning efficiency of value-based methods.  \citet{DoubleQ} proposed to use a second value estimator as cross-validation to address this bias, and this technique has been extended to the deep Q-learning paradigm \citep{van2016deep}. In practice, since it is intractable to construct two fully independent value estimators, double Q-learning is observed to overestimate sometimes, especially in continuous domains. To address this concern, \citet{TD3} proposed clipped double Q-learning to repress further the incentive of overestimation, which has become the default implementation of most advanced approaches \citep{SAC, kalashnikov2018qt}. Recently, to control the estimation bias more precisely, people utilize an ensemble of value networks \citep{MaxminDQN,  TQC,  REDQ}, which causes high computational costs.

\paragraph{Multi-step Bootstrapping.}
The implementation of our proposed algorithm is also related to a technique named multi-step bootstrapping. This branch originates from the idea of eligibility traces \citep{singh1996reinforcement,sutton1988learning}. Recently, it is prevalent in policy gradient methods such as GAE \citep{schulman2015high} and its variants \citep{touati2018convergent}. In the context of temporal-difference learning, multi-step bootstrapping is also beneficial to improving the stability of Q-learning algorithms \citep{van2018deep}. The main limitation of multi-step bootstrapping is that the performance of multi-step training is sensitive to the bootstrapping horizon choice, and the trajectory used is required to be at least nearly on-policy rather than an arbitrary one.

\section{Conclusion}
This work presents Generalizable Episodic Memory, an effective memory-based method that aggregates different experiences from similar states and future consequences. We perform implicit planning by taking the maximum over all possible combinatorial trajectories in the memory and reduces overestimation error by using twin networks.\par 
We provide theoretical analysis to show that our objective does not overestimate in general and converges to $Q^*$ under mild conditions. Experimental results on continuous control tasks show that our method outperforms state-of-the-art model-free RL methods, as well as episodic memory-based algorithms. Our method also demonstrates general applicability in the discrete domain, such as Atari games. Our method is not primarily designed for discrete domains but can still significantly improve the learning efficiency of RL agents under this setting.\par
We make the first step to endow episodic memory with generalization ability. Besides, generalization ability also relies highly on the representation of states and actions. We leave the study of representation learning in episodic memory as interesting future work.

\section{Acknowledgement}
This work is supported in part by Science and Technology Innovation 2030 – ``New Generation Artificial Intelligence'' Major Project (No. 2018AAA0100904), and a grant from the Institute of Guo Qiang,Tsinghua University.

\bibliography{gem_new}

\begin{thebibliography}{46}
\providecommand{\natexlab}[1]{#1}
\providecommand{\url}[1]{\texttt{#1}}
\expandafter\ifx\csname urlstyle\endcsname\relax
  \providecommand{\doi}[1]{doi: #1}\else
  \providecommand{\doi}{doi: \begingroup \urlstyle{rm}\Url}\fi

\bibitem[Bellemare et~al.(2013)Bellemare, Naddaf, Veness, and Bowling]{Atari}
Bellemare, M.~G., Naddaf, Y., Veness, J., and Bowling, M.
\newblock The arcade learning environment: An evaluation platform for general
  agents.
\newblock \emph{Journal of Artificial Intelligence Research}, 47:\penalty0
  253--279, 2013.

\bibitem[Bellman(1957)]{bellman}
Bellman, R.
\newblock Dynamic programming princeton university press princeton.
\newblock \emph{New Jersey Google Scholar}, 1957.

\bibitem[Blundell et~al.(2016)Blundell, Uria, Pritzel, Li, Ruderman, Leibo,
  Rae, Wierstra, and Hassabis]{MFEC}
Blundell, C., Uria, B., Pritzel, A., Li, Y., Ruderman, A., Leibo, J.~Z., Rae,
  J., Wierstra, D., and Hassabis, D.
\newblock Model-free episodic control.
\newblock \emph{arXiv preprint arXiv:1606.04460}, 2016.

\bibitem[Botvinick et~al.(2019)Botvinick, Ritter, Wang, Kurth-Nelson, Blundell,
  and Hassabis]{fastslow}
Botvinick, M., Ritter, S., Wang, J.~X., Kurth-Nelson, Z., Blundell, C., and
  Hassabis, D.
\newblock Reinforcement learning, fast and slow.
\newblock \emph{Trends in cognitive sciences}, 23\penalty0 (5):\penalty0
  408--422, 2019.

\bibitem[Brockman et~al.(2016)Brockman, Cheung, Pettersson, Schneider,
  Schulman, Tang, and Zaremba]{brockman2016openai}
Brockman, G., Cheung, V., Pettersson, L., Schneider, J., Schulman, J., Tang,
  J., and Zaremba, W.
\newblock Openai gym.
\newblock \emph{arXiv preprint arXiv:1606.01540}, 2016.

\bibitem[Chen et~al.(2021)Chen, Wang, Zhou, and Ross]{REDQ}
Chen, X., Wang, C., Zhou, Z., and Ross, K.
\newblock Randomized ensembled double q-learning: Learning fast without a
  model.
\newblock In \emph{International Conference on Learning Representations}, 2021.

\bibitem[Fujimoto et~al.(2018)Fujimoto, van Hoof, and Meger]{TD3}
Fujimoto, S., van Hoof, H., and Meger, D.
\newblock Addressing function approximation error in actor-critic methods.
\newblock In Dy, J.~G. and Krause, A. (eds.), \emph{Proceedings of the 35th
  International Conference on Machine Learning, {ICML} 2018}, volume~80 of
  \emph{Proceedings of Machine Learning Research}, pp.\  1582--1591. {PMLR},
  2018.

\bibitem[Gilboa \& Schmeidler(1995)Gilboa and Schmeidler]{gilboa1995case}
Gilboa, I. and Schmeidler, D.
\newblock Case-based decision theory.
\newblock \emph{The Quarterly Journal of Economics}, 110\penalty0 (3):\penalty0
  605--639, 1995.

\bibitem[Haarnoja et~al.(2018)Haarnoja, Zhou, Abbeel, and Levine]{SAC}
Haarnoja, T., Zhou, A., Abbeel, P., and Levine, S.
\newblock Soft actor-critic: Off-policy maximum entropy deep reinforcement
  learning with a stochastic actor.
\newblock In Dy, J.~G. and Krause, A. (eds.), \emph{Proceedings of the 35th
  International Conference on Machine Learning, {ICML} 2018}, volume~80 of
  \emph{Proceedings of Machine Learning Research}, pp.\  1856--1865. {PMLR},
  2018.

\bibitem[Hansen et~al.(2018)Hansen, Pritzel, Sprechmann, Barreto, and
  Blundell]{EVA}
Hansen, S., Pritzel, A., Sprechmann, P., Barreto, A., and Blundell, C.
\newblock Fast deep reinforcement learning using online adjustments from the
  past.
\newblock In Bengio, S., Wallach, H.~M., Larochelle, H., Grauman, K.,
  Cesa{-}Bianchi, N., and Garnett, R. (eds.), \emph{Advances in Neural
  Information Processing Systems 31: Annual Conference on Neural Information
  Processing Systems 2018, NeurIPS 2018}, pp.\  10590--10600, 2018.

\bibitem[He et~al.(2017)He, Liu, Schwing, and Peng]{he2016learning}
He, F.~S., Liu, Y., Schwing, A.~G., and Peng, J.
\newblock Learning to play in a day: Faster deep reinforcement learning by
  optimality tightening.
\newblock In \emph{5th International Conference on Learning Representations,
  {ICLR} 2017}. OpenReview.net, 2017.

\bibitem[Hessel et~al.(2018)Hessel, Modayil, van Hasselt, Schaul, Ostrovski,
  Dabney, Horgan, Piot, Azar, and Silver]{rainbow}
Hessel, M., Modayil, J., van Hasselt, H., Schaul, T., Ostrovski, G., Dabney,
  W., Horgan, D., Piot, B., Azar, M.~G., and Silver, D.
\newblock Rainbow: Combining improvements in deep reinforcement learning.
\newblock In McIlraith, S.~A. and Weinberger, K.~Q. (eds.), \emph{Proceedings
  of the Thirty-Second {AAAI} Conference on Artificial Intelligence,
  (AAAI-18)}, pp.\  3215--3222. {AAAI} Press, 2018.

\bibitem[Kahn et~al.(1981)]{kahn1981art}
Kahn, C.~H. et~al.
\newblock \emph{The art and thought of Heraclitus: a new arrangement and
  translation of the Fragments with literary and philosophical commentary}.
\newblock Cambridge University Press, 1981.

\bibitem[Kalashnikov et~al.(2018)Kalashnikov, Irpan, Pastor, Ibarz, Herzog,
  Jang, Quillen, Holly, Kalakrishnan, Vanhoucke, et~al.]{kalashnikov2018qt}
Kalashnikov, D., Irpan, A., Pastor, P., Ibarz, J., Herzog, A., Jang, E.,
  Quillen, D., Holly, E., Kalakrishnan, M., Vanhoucke, V., et~al.
\newblock Qt-opt: Scalable deep reinforcement learning for vision-based robotic
  manipulation.
\newblock In \emph{Conference on Robot Learning}. PMLR, 2018.

\bibitem[Kuznetsov et~al.(2020)Kuznetsov, Shvechikov, Grishin, and Vetrov]{TQC}
Kuznetsov, A., Shvechikov, P., Grishin, A., and Vetrov, D.~P.
\newblock Controlling overestimation bias with truncated mixture of continuous
  distributional quantile critics.
\newblock In \emph{Proceedings of the 37th International Conference on Machine
  Learning, {ICML} 2020}, volume 119 of \emph{Proceedings of Machine Learning
  Research}, pp.\  5556--5566. {PMLR}, 2020.

\bibitem[Lan et~al.(2020)Lan, Pan, Fyshe, and White]{MaxminDQN}
Lan, Q., Pan, Y., Fyshe, A., and White, M.
\newblock Maxmin q-learning: Controlling the estimation bias of q-learning.
\newblock In \emph{8th International Conference on Learning Representations,
  {ICLR} 2020}. OpenReview.net, 2020.

\bibitem[Lee et~al.(2019)Lee, Choi, and Chung]{EBU}
Lee, S.~Y., Choi, S., and Chung, S.
\newblock Sample-efficient deep reinforcement learning via episodic backward
  update.
\newblock In Wallach, H.~M., Larochelle, H., Beygelzimer, A.,
  d'Alch{\'{e}}{-}Buc, F., Fox, E.~B., and Garnett, R. (eds.), \emph{Advances
  in Neural Information Processing Systems 32: Annual Conference on Neural
  Information Processing Systems 2019, NeurIPS 2019}, pp.\  2110--2119, 2019.

\bibitem[Lengyel \& Dayan(2007)Lengyel and Dayan]{lengyel2007hippocampal}
Lengyel, M. and Dayan, P.
\newblock Hippocampal contributions to control: The third way.
\newblock In Platt, J.~C., Koller, D., Singer, Y., and Roweis, S.~T. (eds.),
  \emph{Advances in Neural Information Processing Systems 20, Proceedings of
  the Twenty-First Annual Conference on Neural Information Processing Systems},
  pp.\  889--896. Curran Associates, Inc., 2007.

\bibitem[Lillicrap et~al.(2016)Lillicrap, Hunt, Pritzel, Heess, Erez, Tassa,
  Silver, and Wierstra]{DDPG}
Lillicrap, T.~P., Hunt, J.~J., Pritzel, A., Heess, N., Erez, T., Tassa, Y.,
  Silver, D., and Wierstra, D.
\newblock Continuous control with deep reinforcement learning.
\newblock In Bengio, Y. and LeCun, Y. (eds.), \emph{4th International
  Conference on Learning Representations, {ICLR} 2016,}, 2016.

\bibitem[Lin et~al.(2018)Lin, Zhao, Yang, and Zhang]{EMDQN}
Lin, Z., Zhao, T., Yang, G., and Zhang, L.
\newblock Episodic memory deep q-networks.
\newblock In Lang, J. (ed.), \emph{Proceedings of the Twenty-Seventh
  International Joint Conference on Artificial Intelligence, {IJCAI} 2018},
  pp.\  2433--2439. ijcai.org, 2018.
\newblock \doi{10.24963/ijcai.2018/337}.

\bibitem[Machado et~al.(2018)Machado, Bellemare, Talvitie, Veness, Hausknecht,
  and Bowling]{machado2018revisiting}
Machado, M.~C., Bellemare, M.~G., Talvitie, E., Veness, J., Hausknecht, M., and
  Bowling, M.
\newblock Revisiting the arcade learning environment: Evaluation protocols and
  open problems for general agents.
\newblock \emph{Journal of Artificial Intelligence Research}, 61:\penalty0
  523--562, 2018.

\bibitem[Marr et~al.(1991)Marr, Willshaw, and McNaughton]{marr1991simple}
Marr, D., Willshaw, D., and McNaughton, B.
\newblock Simple memory: a theory for archicortex.
\newblock In \emph{From the Retina to the Neocortex}, pp.\  59--128. Springer,
  1991.

\bibitem[Mnih et~al.(2015)Mnih, Kavukcuoglu, Silver, Rusu, Veness, Bellemare,
  Graves, Riedmiller, Fidjeland, Ostrovski, et~al.]{DQN}
Mnih, V., Kavukcuoglu, K., Silver, D., Rusu, A.~A., Veness, J., Bellemare,
  M.~G., Graves, A., Riedmiller, M., Fidjeland, A.~K., Ostrovski, G., et~al.
\newblock Human-level control through deep reinforcement learning.
\newblock \emph{nature}, 518\penalty0 (7540):\penalty0 529--533, 2015.

\bibitem[Oh et~al.(2018)Oh, Guo, Singh, and Lee]{SIL}
Oh, J., Guo, Y., Singh, S., and Lee, H.
\newblock Self-imitation learning.
\newblock In Dy, J.~G. and Krause, A. (eds.), \emph{Proceedings of the 35th
  International Conference on Machine Learning, {ICML} 2018}, volume~80 of
  \emph{Proceedings of Machine Learning Research}, pp.\  3875--3884. {PMLR},
  2018.

\bibitem[Ormoneit \& Sen(2002)Ormoneit and Sen]{ormoneit2002kernel}
Ormoneit, D. and Sen, {\'S}.
\newblock Kernel-based reinforcement learning.
\newblock \emph{Machine learning}, 49\penalty0 (2):\penalty0 161--178, 2002.

\bibitem[Peters \& Bagnell(2010)Peters and Bagnell]{PetersAC}
Peters, J. and Bagnell, J.~A.
\newblock Policy gradient methods.
\newblock \emph{Scholarpedia}, 5\penalty0 (11):\penalty0 3698, 2010.

\bibitem[Pritzel et~al.(2017)Pritzel, Uria, Srinivasan, Badia, Vinyals,
  Hassabis, Wierstra, and Blundell]{NEC}
Pritzel, A., Uria, B., Srinivasan, S., Badia, A.~P., Vinyals, O., Hassabis, D.,
  Wierstra, D., and Blundell, C.
\newblock Neural episodic control.
\newblock In Precup, D. and Teh, Y.~W. (eds.), \emph{Proceedings of the 34th
  International Conference on Machine Learning, {ICML} 2017}, volume~70 of
  \emph{Proceedings of Machine Learning Research}, pp.\  2827--2836. {PMLR},
  2017.

\bibitem[Schulman et~al.(2016)Schulman, Moritz, Levine, Jordan, and
  Abbeel]{schulman2015high}
Schulman, J., Moritz, P., Levine, S., Jordan, M.~I., and Abbeel, P.
\newblock High-dimensional continuous control using generalized advantage
  estimation.
\newblock In Bengio, Y. and LeCun, Y. (eds.), \emph{4th International
  Conference on Learning Representations, {ICLR} 2016}, 2016.

\bibitem[Schulman et~al.(2017)Schulman, Wolski, Dhariwal, Radford, and
  Klimov]{PPO}
Schulman, J., Wolski, F., Dhariwal, P., Radford, A., and Klimov, O.
\newblock Proximal policy optimization algorithms.
\newblock \emph{arXiv preprint arXiv:1707.06347}, 2017.

\bibitem[Shohamy \& Wagner(2008)Shohamy and Wagner]{shohamy2008integrating}
Shohamy, D. and Wagner, A.~D.
\newblock Integrating memories in the human brain: hippocampal-midbrain
  encoding of overlapping events.
\newblock \emph{Neuron}, 60\penalty0 (2):\penalty0 378--389, 2008.

\bibitem[Silver et~al.(2014)Silver, Lever, Heess, Degris, Wierstra, and
  Riedmiller]{DPG}
Silver, D., Lever, G., Heess, N., Degris, T., Wierstra, D., and Riedmiller,
  M.~A.
\newblock Deterministic policy gradient algorithms.
\newblock In \emph{Proceedings of the 31th International Conference on Machine
  Learning, {ICML} 2014}, volume~32 of \emph{{JMLR} Workshop and Conference
  Proceedings}, pp.\  387--395. JMLR.org, 2014.

\bibitem[Silver et~al.(2016)Silver, Huang, Maddison, Guez, Sifre, Van
  Den~Driessche, Schrittwieser, Antonoglou, Panneershelvam, Lanctot,
  et~al.]{AlphaGo}
Silver, D., Huang, A., Maddison, C.~J., Guez, A., Sifre, L., Van Den~Driessche,
  G., Schrittwieser, J., Antonoglou, I., Panneershelvam, V., Lanctot, M.,
  et~al.
\newblock Mastering the game of go with deep neural networks and tree search.
\newblock \emph{nature}, 529\penalty0 (7587):\penalty0 484--489, 2016.

\bibitem[Singh \& Sutton(1996)Singh and Sutton]{singh1996reinforcement}
Singh, S.~P. and Sutton, R.~S.
\newblock Reinforcement learning with replacing eligibility traces.
\newblock \emph{Machine learning}, 22\penalty0 (1):\penalty0 123--158, 1996.

\bibitem[Sutherland \& Rudy(1989)Sutherland and Rudy]{sutherland1989configural}
Sutherland, R.~J. and Rudy, J.~W.
\newblock Configural association theory: The role of the hippocampal formation
  in learning, memory, and amnesia.
\newblock \emph{Psychobiology}, 17\penalty0 (2):\penalty0 129--144, 1989.

\bibitem[Sutton(1988{\natexlab{a}})]{TDLearning}
Sutton, R.~S.
\newblock Learning to predict by the methods of temporal differences.
\newblock \emph{Machine learning}, 3\penalty0 (1):\penalty0 9--44,
  1988{\natexlab{a}}.

\bibitem[Sutton(1988{\natexlab{b}})]{sutton1988learning}
Sutton, R.~S.
\newblock Learning to predict by the methods of temporal differences.
\newblock \emph{Machine learning}, 3\penalty0 (1):\penalty0 9--44,
  1988{\natexlab{b}}.

\bibitem[Sutton et~al.(1999)Sutton, McAllester, Singh, Mansour,
  et~al.]{SuttonAC}
Sutton, R.~S., McAllester, D.~A., Singh, S.~P., Mansour, Y., et~al.
\newblock Policy gradient methods for reinforcement learning with function
  approximation.
\newblock In \emph{NIPS}, volume~99, pp.\  1057--1063. Citeseer, 1999.

\bibitem[Thrun \& Schwartz(1993)Thrun and Schwartz]{thrun1993issues}
Thrun, S. and Schwartz, A.
\newblock Issues in using function approximation for reinforcement learning.
\newblock In \emph{Proceedings of the Fourth Connectionist Models Summer
  School}, pp.\  255--263. Hillsdale, NJ, 1993.

\bibitem[Todorov et~al.(2012)Todorov, Erez, and Tassa]{mujoco}
Todorov, E., Erez, T., and Tassa, Y.
\newblock Mujoco: A physics engine for model-based control.
\newblock In \emph{2012 IEEE/RSJ International Conference on Intelligent Robots
  and Systems}, pp.\  5026--5033. IEEE, 2012.

\bibitem[Touati et~al.(2018)Touati, Bacon, Precup, and
  Vincent]{touati2018convergent}
Touati, A., Bacon, P., Precup, D., and Vincent, P.
\newblock Convergent {TREE} {BACKUP} and {RETRACE} with function approximation.
\newblock In Dy, J.~G. and Krause, A. (eds.), \emph{Proceedings of the 35th
  International Conference on Machine Learning, {ICML} 2018}, volume~80 of
  \emph{Proceedings of Machine Learning Research}, pp.\  4962--4971. {PMLR},
  2018.

\bibitem[Tsividis et~al.(2017)Tsividis, Pouncy, Xu, Tenenbaum, and
  Gershman]{DBLP:conf/aaaiss/TsividisPXTG17}
Tsividis, P., Pouncy, T., Xu, J.~L., Tenenbaum, J.~B., and Gershman, S.~J.
\newblock Human learning in atari.
\newblock In \emph{2017 {AAAI}}. {AAAI} Press, 2017.

\bibitem[van Hasselt(2010)]{DoubleQ}
van Hasselt, H.
\newblock Double q-learning.
\newblock In Lafferty, J.~D., Williams, C. K.~I., Shawe{-}Taylor, J., Zemel,
  R.~S., and Culotta, A. (eds.), \emph{Advances in Neural Information
  Processing Systems 23: 24th Annual Conference on Neural Information
  Processing Systems 2010}, pp.\  2613--2621. Curran Associates, Inc., 2010.

\bibitem[van Hasselt et~al.(2016)van Hasselt, Guez, and Silver]{van2016deep}
van Hasselt, H., Guez, A., and Silver, D.
\newblock Deep reinforcement learning with double q-learning.
\newblock In Schuurmans, D. and Wellman, M.~P. (eds.), \emph{Proceedings of the
  Thirtieth {AAAI} Conference on Artificial Intelligence}, pp.\  2094--2100.
  {AAAI} Press, 2016.

\bibitem[Van~Hasselt et~al.(2018)Van~Hasselt, Doron, Strub, Hessel, Sonnerat,
  and Modayil]{van2018deep}
Van~Hasselt, H., Doron, Y., Strub, F., Hessel, M., Sonnerat, N., and Modayil,
  J.
\newblock Deep reinforcement learning and the deadly triad.
\newblock \emph{arXiv preprint arXiv:1812.02648}, 2018.

\bibitem[Wang et~al.(2016)Wang, Schaul, Hessel, van Hasselt, Lanctot, and
  de~Freitas]{wang2016dueling}
Wang, Z., Schaul, T., Hessel, M., van Hasselt, H., Lanctot, M., and de~Freitas,
  N.
\newblock Dueling network architectures for deep reinforcement learning.
\newblock In Balcan, M. and Weinberger, K.~Q. (eds.), \emph{Proceedings of the
  33nd International Conference on Machine Learning, {ICML} 2016}, volume~48 of
  \emph{{JMLR} Workshop and Conference Proceedings}, pp.\  1995--2003.
  JMLR.org, 2016.

\bibitem[Zhu et~al.(2020)Zhu, Lin, Yang, and Zhang]{associative}
Zhu, G., Lin, Z., Yang, G., and Zhang, C.
\newblock Episodic reinforcement learning with associative memory.
\newblock In \emph{8th International Conference on Learning Representations,
  {ICLR} 2020, Addis Ababa, Ethiopia, April 26-30, 2020}. OpenReview.net, 2020.

\end{thebibliography}
\bibliographystyle{icml2021}
\newpage
\appendix
\onecolumn
\icmltitle{Supplementary Material}
\icmlsetsymbol{equal}{*}
\vskip 0.3in

\section{GEM Algorithm in Tabular Case}
\label{tabular}
In this section, we present the formal description of the GEM algorithm in tabular case, as shown in Algorithm \ref{alg:tabular}.
\begin{algorithm}[H]
    \caption{Generalizable Episodic Memory in Tabular Case}
    \label{alg:tabular}
    \begin{algorithmic}
        \STATE Initialize table $Q^{(1)}(s,a), Q^{(2)}(s,a)$ arbitrarily,
        \STATE Initial learning step size $\alpha_t$, small $\epsilon>0$ and episode length $l=0$
        \STATE Set $\pi$ to be the $\epsilon$-greedy policy with respect to $Q^{(1)}(s,a)$ or $Q^{(2)}(s,a)$
        \FOR{$t=1,\cdots,$}
            \STATE Initialize and store $s_0$
            \STATE Select action $a_0 \sim \pi(\cdot|s_0)$\;
            \STATE Observe reward $r$ and new state $s'$\;
            \STATE Store transition tuple $(s, a, r, s')$\;
            \STATE $l\leftarrow l+1$
            \IF{ an episode is ended}
                \FOR{ $\tau=t-l,\cdots,t$ }
                    \STATE Compute $R^{(1)}_\tau,R^{(2)}_\tau$ according to Equation~(\ref{update1})\;
                    \STATE Uniformly choose $i\in\{1,2\}$\;
                    \STATE Update $Q^{(i)}(s_\tau,a_\tau) \leftarrow Q^{(i)}(s_\tau,a_\tau)+\alpha_\tau(R_\tau^{(i)}-Q^{(i)}(s_\tau,a_\tau))$\;
                \ENDFOR
                \STATE Set $\pi$ to be the $\epsilon$-greedy policy with respect to $Q^{(1)}(s,a)$ or $Q^{(2)}(s,a)$
                \STATE $l \leftarrow 0$
            \ENDIF
        \ENDFOR
    \end{algorithmic}

\end{algorithm}

\section{Proofs of Theoremss}
\label{proofs}
\begin{theorem}
    Given unbiased and independent estimators $Q^{(1,2)}_\theta(s_{t+h},a_{t+h})=Q^\pi(s_{t+h},a_{t+h})+\epsilon^{(1,2)}_h$, Equation~(\ref{update2}) will not overestimate the true objective, i.e.
    \begin{equation}
        \EX_{\tau,\epsilon} \left[R_t^{(1,2)}(s_t)\right]\leq  \EX_{\tau}\left[\max_{0\leq h\leq T-t-1}Q^\pi_{t,h}(s_t)\right],
    \end{equation}
    where     
    \begin{equation}
        Q^\pi_{t,h}(s,a)=
        \left\{\begin{aligned}
        &\sum_{i=0}^{h}\gamma^i r_{t+i}+\gamma^{h+1} Q^\pi(s_{t+h+1},a_{t+h+1})& \text{if } h<T-t,\\
        &\sum_{i=0}^{h}\gamma^i r_{t+i}& \text{if } h=T-t.
    \end{aligned}
    \right.
    \end{equation}
    and $\tau=\{(s_t,a_t,r_t,s_{t+1})_{t=1,\cdots,T}\}$ is a trajectory.
\end{theorem}
\begin{proof}
    \par
    By unrolling and rewritten Equation~(\ref{update1}), we have 
    \begin{align*}
        R_t^{(1,2)} = V_{t,h^*}=\sum_{i=0}^{h^*}\gamma^i r_{t+i}+\gamma^{h^*+1}Q^{(2,1)}_\theta(s_{t+h^*+1},a_{t+h^*+1}),
    \end{align*}
    Where $h^*$ is the abbrevation for $h^*_{(1,2)}$ for simplicity. Then we have
    \begin{align*}
        \EX_\epsilon\left[ R_t^{(1,2)} - Q^\pi_{t,h^*_{(1,2)}}(s_t)\right]&=\EX\left[V_{t,h^*_{(1,2)}}-Q^\pi_{t,h^*_{(1,2)}}(s_t)\right]\\
        &=\EX\left[\gamma^{h^*+1}\left(Q^{(2,1)}_\theta(s_{t+h^*+1},a_{t+h^*+1})-Q^\pi(s_{t+h^*+1},a_{t+h^*+1})\right)\right]\\
        &=0.\\
    \end{align*}

Then naturally $$\EX_{\tau,\epsilon}[R_t^{(1,2)}] = \EX_{\tau} [Q^\pi_{t,h^*_{(1,2)}}(s_t)]\leq \EX_{\tau} \left[\max_{0\leq h\leq T-t}{Q^\pi_{t,h}(s_t)}\right].$$
\end{proof}

To prepare for the theorem below, we need the following lemma:\par
\begin{lemma}
    \label{lemma1}
    Consider a stochastic process $(\zeta_t,\Delta_t,F_t),t \geq 0$, where $\zeta,\Delta_t,F_t : X \rightarrow \mathbb{R}$ satisfy the equations
    \begin{equation}
        \label{equ:lemma}
        \Delta_{t+1}(x)=(1-\zeta_t(x))\Delta_t(x)+\zeta_t(x) F_t(x)
    \end{equation}
    
    Let $\{P_t\}$ be a filter such that $\zeta_t$ and $\Delta_t$ are $P_t$-measurable, $F_t$ is $P_{t+1}$-measurable, $t\ge 0$. Assume that the following hold:
    
    \begin{itemize}
        \item $X$ is finite: $|X|<+\infty$.
        \item $\zeta_t(x)\in[0,1]$, $\sum_t\zeta_t(x)=+\infty$, $\sum_t\zeta_t^2(x)<+\infty$ a.s. for all $x\in X$.
        \item $\|\mathbb E(F_t|P_t)\|_\infty\le \kappa \|\Delta_t\|_\infty+c_t$, where $\kappa\in[0,1)$ and $c_t\stackrel{\text{a.s}}{\rightarrow}0$.
        \item $\text{Var}(F_t|P_t)\le K(1+\|\Delta_t\|_\infty)^2$, where $K$ is some constant.
    \end{itemize}
    
    Then $\Delta_t$ converge to zero w.p.1.
\end{lemma}
This lemma is also used in Double-Q learning \citep{DoubleQ} and we omit the proof for simplicity. \\
In the following sections, we use $\norm{\cdot}$ to represent the infinity norm $\norm{\cdot}_\infty$.

\begin{theorem}
    \leavevmode
    Algorithm \ref{alg:tabular} converge to $Q^*$ w.p.1 with the following conditions:
    \begin{enumerate}
        \item The MDP is finite, i.e. $|\mathcal{S}\times\mathcal{A}|\leq\infty$
        \item $\gamma \in [0,1)$
        \item The Q-values are stored in a lookup table
        \item $\alpha_t(s,a) \in [0,1],\sum_t{\alpha_t(s,a)}=\infty,\sum_t{\alpha_t^2(s,a)}\leq\infty$
        \item The environment is fully deterministic, i.e. $P(s'|s,a)=\delta(s'=f(s,a))$ for some deterministic transition function $f$
    \end{enumerate}
\end{theorem}

\begin{proof}
    This is a sketch of proof and some technical details are omitted.\par
    We just need to show that without double-q version, the update will be a $\gamma$-contraction and will converge. Then we need to show that $\norm{Q^1-Q^2}$ converge to zero, which is similar with double-q learning.\par
    We only prove convergence of $Q^{(1)}$, and by symmetry we have the conclusion.\par
    Let $\Delta_t = Q_t^{(1)}-Q^*$, and $F_t(s_t,a_t)=R^{(1)}_t-Q^*(s_t,a_t)$,\par
    Then the update rule can be written exactly as  Equation~(\ref{equ:lemma}):
    $$\Delta_{t+1}=(1-\alpha_t)\Delta_t+\alpha_t F_t.$$
    We define $$G_t=\tilde{R}^{(1)}_t-Q^*(s_t,a_t)=F_t+(\tilde{R}^{(1)}_t-R^{(1)}_t),$$\par
    where $\tilde{R}^{(1)}=R^{(1)}_{t,h^*_{(1)}},$and the notation is kept the same as in Equation~(\ref{update1}) \& (\ref{update2}).\par
    To use Lemma \ref{lemma1}, we only need to prove that $G_t$ is a $\gamma$-contraction and $c_t=\tilde{R}^{(1)}_t-R^{(1)}_t$ converge to zero.\par
    On the one hand, 
    \begin{align*}
        \tilde{R}^{(1)}_t-Q^*(s_t,a_t)&\geq r_t+\gamma \tilde{Q}^{(1)}(s_{t+1},\tilde{a}^*) -Q^*(s_t,a_t)\\
        &=r_t+\gamma \tilde{Q}^{(1)}(s_{t+1},\tilde{a}^*) -r_t+\gamma Q^*(s_{t+1},a^*)\\
        &=\gamma (\tilde{Q}^{(1)}(s_{t+1},\tilde{a}^*)-Q^*(s_{t+1},a^*))\\
        &\geq -\gamma\norm{\Delta_t}.\\
    \end{align*}

    On the other hand, 
    \begin{align*}
        \tilde{R}^{(1)}_t-Q^*(s_t,a_t)&=\sum_{i=0}^{h^*_{(1)}}\gamma^i r_{t+i}+\gamma^{h^*_{(1)}+1}\tilde{Q}^{(1)}(s_{t+h^*_{(1)}+1},\tilde{a}^*)- Q^*(s_t,a_t)\\
        &\leq \sum_{i=0}^{h^*_{(1)}}\gamma^i r_{t+i}+\gamma^{h^*_{(1)}+1}\tilde{Q}^{(1)}_\pi(s_{t+h^*_{(1)}+1})\\
        &-(\sum_{i=0}^{h^*_{(1)}}\gamma^i r_{t+i}+\gamma^{h^*_{(1)}+1}Q^*(s_{t+h^*_{(1)}+1},a^*))\\
        &=\gamma^{h^*_{(1)}+1} (\tilde{Q}^{(1)}(s_{t+h^*_{(1)}+1},\tilde{a}^*)-Q^*(s_{t+h^*_{(1)}+1},a^*))\\
        &\leq \gamma (\tilde{Q}^{(1)}(s_{t+h^*_{(1)}+1},\tilde{a}^*)-Q^*(s_{t+h^*_{(1)}+1},a^*))\\
        &\leq \gamma \norm{\Delta_t}.\\
    \end{align*}
    Thus $G_t$ is a $\gamma$-contraction w.r.t $\Delta_t$.\par
    Finally we show $c_t =\tilde{R}^{(1)}-R^{(1)}_t $ converges to zero.\par
    Note that $c_t = \gamma^{h^*_{(1)}}(\tilde{Q}^{(1)}-\tilde{Q}^{(2)})$, it suffices to show that $\Delta^{1,2}=\tilde{Q}^{(1)}-\tilde{Q}^{(2)}$ converge to zero.\par
    Depending on whether $\tilde{Q}^{(1)}$ or $\tilde{Q}^{(2)}$ is updated, the update rule can be written as 
    $$\Delta^{1,2}_{t+1}=\Delta^{1,2}_{t}+\alpha_t F_t^{(2)}(s_t,a_t),$$
    or
    $$\Delta^{1,2}_{t+1}=\Delta^{1,2}_{t}-\alpha_t F_t^{(1)}(s_t,a_t),$$
    where
    $F_t^{(1)}=R_t^{(1)}-\tilde{Q}_t^{(2)}$ and $F_t^{(2)}=R_t^{(2)}-\tilde{Q}_t^{(1)}$.\par
    Now let $\zeta_t=\frac{1}{2}\alpha_t$, we have
    \begin{align*}
       \EX[\Delta^{1,2}_{t+1}|P_t]&=\frac{1}{2} (\Delta^{1,2}+\alpha_t \EX [F_t^{(2)}])+\frac{1}{2} (\Delta^{1,2}-\alpha_t \EX[F_t^{(1)}])\\
       &=(1-\zeta_t)\Delta^{1,2}_t+\zeta_t \EX[R_t^{(2)}-R_t^{(1)}] \\
    \end{align*}
    
    when $\EX[R_t^{(2)}] \geq \EX[R_t^{(1)}],$
    by definition we have $\EX[R_t^{(2)}]\leq\EX[\tilde{R}_t^{(2)}].$\par
    Then 
    \begin{align*}
        |\EX[R_t^{(2)}-R_t^{(1)}]|&\leq \EX[\tilde{R}_t^{(2)}-R_t^{(1)}]\\
        &\leq \gamma^{h^*_{(2)}+1}(Q^{(1)}(s_{t+h^*_{(2)}+1},a^*_{(1)})-Q^{(2)}(s_{t+h^*_{(2)}+1},a^*_{(1)}))\\
        &\leq \gamma \norm{\Delta^{1,2}_t}.\\
    \end{align*}

    Similarly, $\EX[R_t^{(2)}] < \EX[R_t^{(1)}],$, we have
    \begin{align*}
        |\EX[R_t^{(2)}-R_t^{(1)}]|&\leq \EX[\tilde{R}_t^{(1)}-R_t^{(2)}]\\
        &\leq \gamma^{h^*_{(1)}+1}(Q^{(2)}(s_{t+h^*_{(1)}+1},a^*_{(2)})-Q^{(1)}(s_{t+h^*_{(1)}+1},a^*_{(2)}))\\
        &\leq \gamma \norm{\Delta^{1,2}_t}.\\
    \end{align*}

    Now in both scenairos we have $|E\{F^{(1,2)}_t|P_t\}|\leq\gamma \norm{\Delta^{1,2}_t}$ holds. Applying Lemma \ref{lemma1} again we have the desired results.
\end{proof}
The theorem apply only to deterministic scenairos. Nevertheless, we can still bound the performance when the environment is stochastic but nearly deterministic.\par
\begin{theorem}
    $\tilde{Q}(s,a)$ learned by Algorithm \ref{alg:tabular} satisfy the following inequality:
    \begin{equation}
        \forall s\in \mathcal{S},a\in\mathcal{A},Q^*(s,a)\leq \tilde{Q}(s,a) \leq Q_{\text{max}}(s,a),
    \end{equation}
    w.p.1 with condition 1-4 in Theorem \ref{theorem2}.
\end{theorem}
\begin{proof}
    We just need to prove that $(Q^*-Q^{(1,2)})_+$ and $(Q^{(1,2)}-Q_{\text{max}})_+$ converge to 0 w.p.1, where $(\cdot)_+=\max(0,\cdot)$.\par
    On the one hand, similar from the proof of Theorem \ref{theorem2} and let $\Delta_t = (Q^*(s_t,a_t)-Q^{(1,2)}(s_t,a_t))_+$.
    \begin{align*}
        Q^*(s_t,a_t)-\tilde{R}^{(1,2)}_t&\leq Q^*(s_t,a_t) - (r_t+\gamma \tilde{Q}^{(1,2)}(s_{t+1},\tilde{a}^*)) \\
        &=r_t+\gamma Q^*(s_{t+1},a^*) -r_t-\gamma \tilde{Q}^{(1,2)}(s_{t+1},\tilde{a}^*)\\
        &=\gamma (\tilde{Q}(s_{t+1},\tilde{a}^*)-Q^*(s_{t+1},a^*))\\
        &\leq \gamma\norm{\Delta_t}.\\
    \end{align*}

    The rest is the same as the proof of Theorem \ref{theorem2}, and we have $(Q^*-Q^{(1,2)})_+$ converge to zero w.p.1.\par
    On the other hand, let $\Delta_t = (Q_{\text{max}}(s_t,a_t)-Q^{(1,2)}(s_t,a_t))_+$,\par
    We have
    \begin{align*}
        F_{t+1}&=\tilde{R}^{(1,2)}_t-Q^{max}_t\\
        &\leq\sum^{h^*_{(2,1)}}_{i=0}\gamma^i r_{t+i}+\gamma^{h^*+1}\tilde{Q}^{(1,2)}_{t+h^*_{(2,1)+1}}-(\sum^{h^*_{(2,1)}}_{i=0}\gamma^i r_{t+i}+\gamma^{h^*+1}Q^{max}_{t+h^*_{(2,1)+1}})\\
        &\leq \gamma^{h^*+1}(\tilde{Q}^{(1,2)}_{t+h^*_{(2,1)+1}}-Q^{max}_{t+h^*_{(2,1)+1}})\\
        &\leq \gamma \norm{\Delta_t}.\\
    \end{align*}
    
    The rest is the same as the proof of Theorem \ref{theorem2}, and we have $(Q^{(1,2)}-Q_{\text{max}})_+$ converge to zero w.p.1.\par
\end{proof}
When the enironment is nearly-deterministic, we can bound the performance of Q despite its non-convergence:
\begin{theorem}
    For a nearly-deterministic environment with factor $\mu$, in limit, GEM's performance can be bounded by 
    \begin{equation}
        V^{\tilde{\pi}}(s)\geq V^*(s)-\frac{2\mu}{1-\gamma},\forall s \in \mathcal{S}.
    \end{equation}
\end{theorem}
\begin{proof}
 since we have $\norm{\tilde{Q}-Q^*} \leq \mu$,
 It is easy to show that 
 \begin{align*}
    &V^*(s)-V^{\tilde{\pi}}(s)\\
    &=Q^*(s,a^*)-Q^{\tilde{\pi}}(s,\tilde{a})\\
    &=Q^*(s,a^*)-\tilde{Q}(s,a^*)+\tilde{Q}(s,a^*)-Q^{\tilde{\pi}}(s,\tilde{a})\\
    &\leq \epsilon +\tilde{Q}(s,\tilde{a})-Q^{\tilde{\pi}}(s,\tilde{a})\\
    &=\epsilon + (\tilde{Q}(s,\tilde{a})-Q^*(s,\tilde{a}))+(Q^*(s,\tilde{a})-Q^{\tilde{\pi}}(s,\tilde{a}))\\
    &\leq 2\epsilon + \gamma (V^*(s)-V^{\tilde{\pi}}(s)).
\end{align*}

So we have the conclusion.
\end{proof}

\section{Hyperparameters}
\label{hyperparameters}
Here we listed the hyperparameters we used for the evaluation of our algorithm.
%
\begin{table}[H]
    \centering
    \begin{tabular}{@{}l|llllll@{}}
    \toprule
    Task           & HalfCheetah & Ant  & Swimmer & Humanoid & Walker & Hopper \\ \midrule
    Maximum Length $d$ & 1000        & 1000 & 1000    & 5        & 5      & 5      \\ \bottomrule
    \end{tabular}
    \caption{Maximum length of rollouts used in GEM across different tasks}
\end{table}

\begin{table}[H]
    \centering
    \begin{tabular}{cc}
    \toprule
    \textbf{Hyper-parameter}            & \textbf{GEM}                  \\ 
    \midrule
    Critic Learning Rate       & 1e-3                 \\
    Actor Learning Rate        & 1e-3                 \\
    Optimizer                  & Adam                 \\
    Target Update Rate($\tau$) & 0.6                  \\
    Memory Update Period($u$)  & 100                \\
    Memory Size  & 100000                             \\
    Policy Delay($p$)          & 2                    \\
    Batch Size                 & 100                  \\
    Discount Factor            & 0.99                 \\
    Exploration Policy         & $\mathcal{N}(0,0.1)$ \\
    Gradient Steps per Update  & 200                  \\ 
    \bottomrule
    \end{tabular}
    \caption{List of Hyperparameters used in GEM across different tasks}

\end{table}
The hyper-parameters for Atari games are kept the same as in the continuous domain, and other hyper-parameters are kept the same as Rainbow \citep{rainbow}.
\section{Additional Ablation Results}
\label{additional}
Here we include more ablation results of GEM. To verify the effectiveness of our proposed implicit planning, we compare our method with simple n-step Q learning combined with TD3. For a fair comparison, we include all different rollout lengths used in GEM's result. The result is shown in Figure \ref{nstep}. We can see that GEM significantly outperform simple n-step learning.\par
To understand the effects of rollout lengths, we also compare the result of different rollout lengths on Atari games. The result is shown below in Figure \ref{atari_ablation}. We can see that using short rollout length greatly hinders the performance of GEM.\par

To verify the effectiveness of GEM on the stochastic domain, we conduct experiments on Atari games with sticky actions, as suggested in \citep{machado2018revisiting}. As illustrated in Figure \ref{atari_sticky}, GEM is still competitive on stochastic domains.

\begin{figure*}[hb]
    \centering
    \includegraphics[width=15cm]{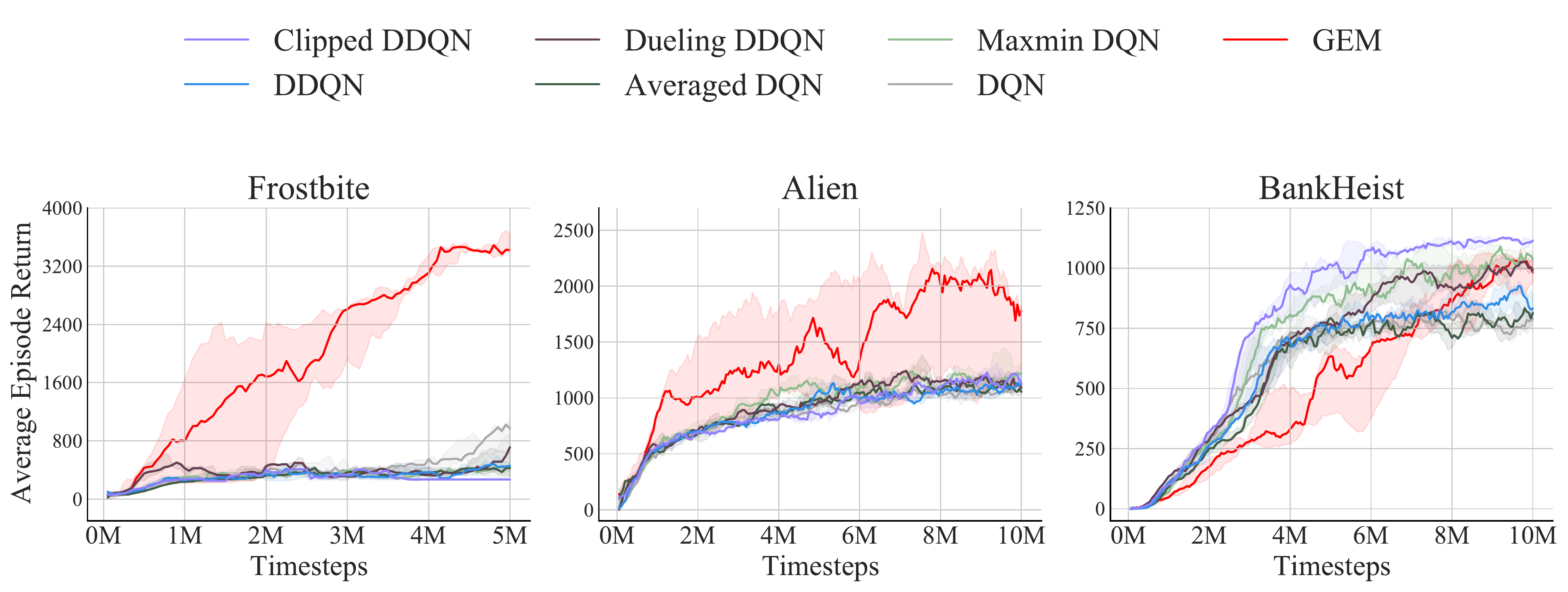}
    \caption{Comparison on 3 Atari games, with sticky actions to make the environment stochastic.}
    \label{atari_sticky}
\end{figure*}

\begin{figure*}[ht]
    \centering
    \begin{subfigure}[c]{0.3\textwidth}
    \centering
    \includegraphics[width=0.98\linewidth]{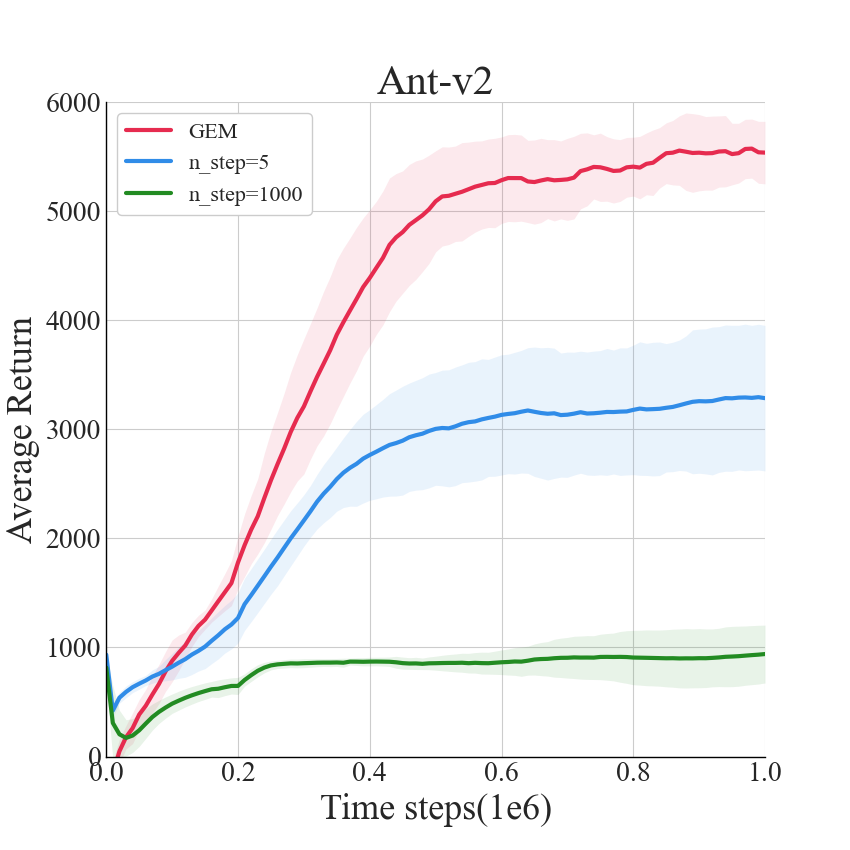}
    \end{subfigure}
    \begin{subfigure}[c]{0.3\textwidth}
    \centering
    \includegraphics[width=0.98\linewidth]{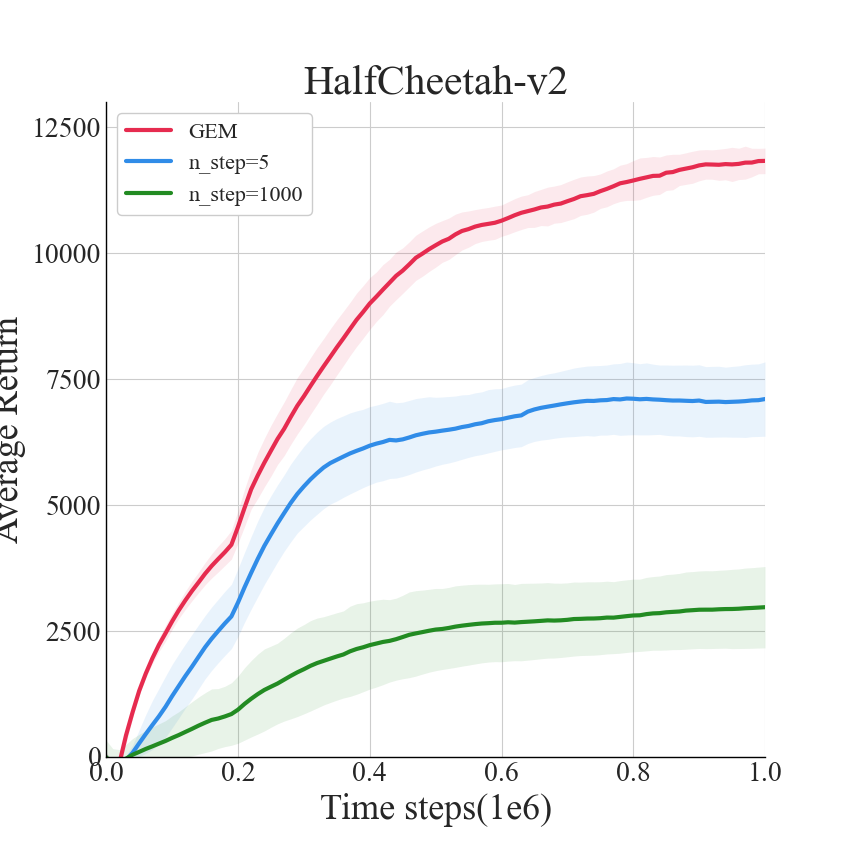}
    \end{subfigure}
    \begin{subfigure}[c]{0.3\textwidth}
    \centering
    \includegraphics[width=0.98\linewidth]{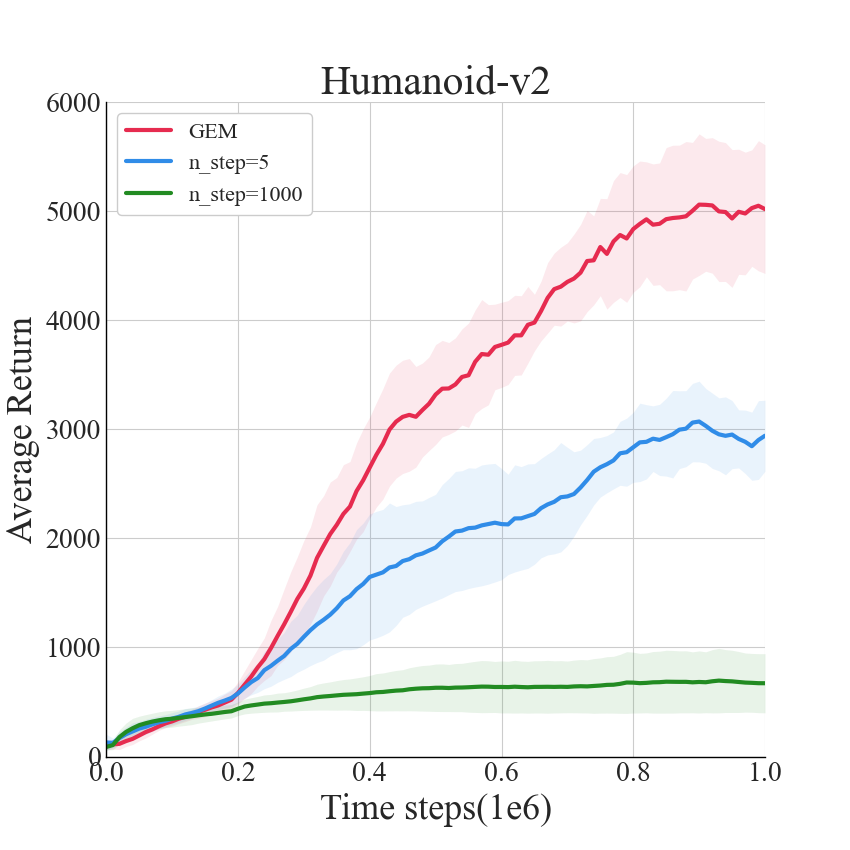}
    \end{subfigure}
    \newline
    \begin{subfigure}[c]{0.3\textwidth}
     \centering
     \includegraphics[width=0.98\linewidth]{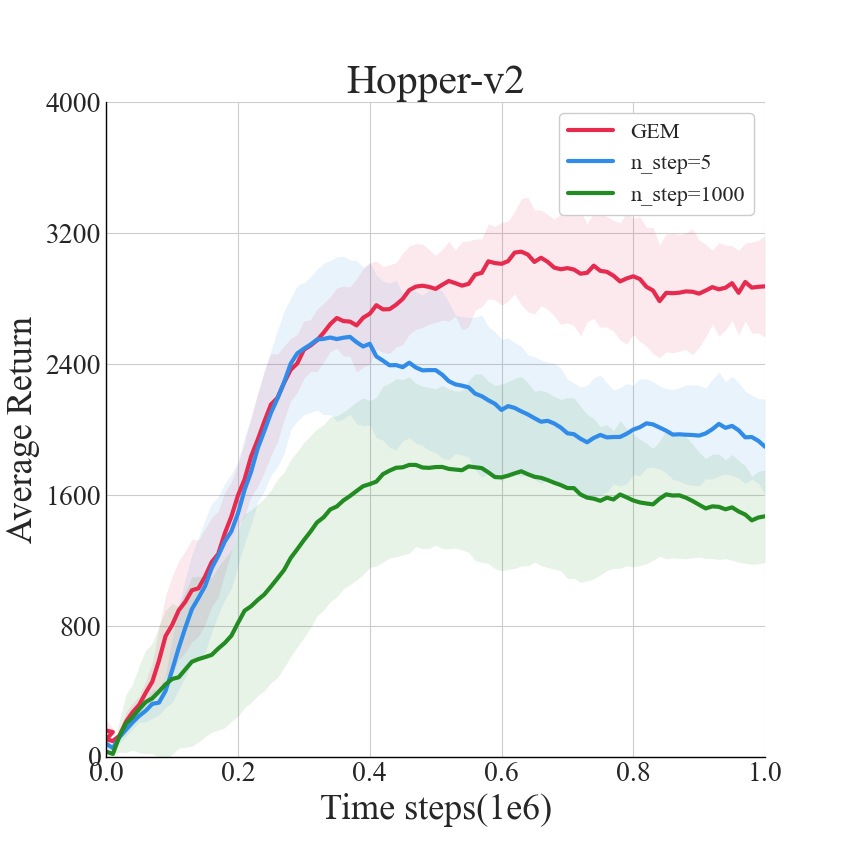}
    \end{subfigure}
    \begin{subfigure}[c]{0.3\textwidth}
     \centering
     \includegraphics[width=0.98\linewidth]{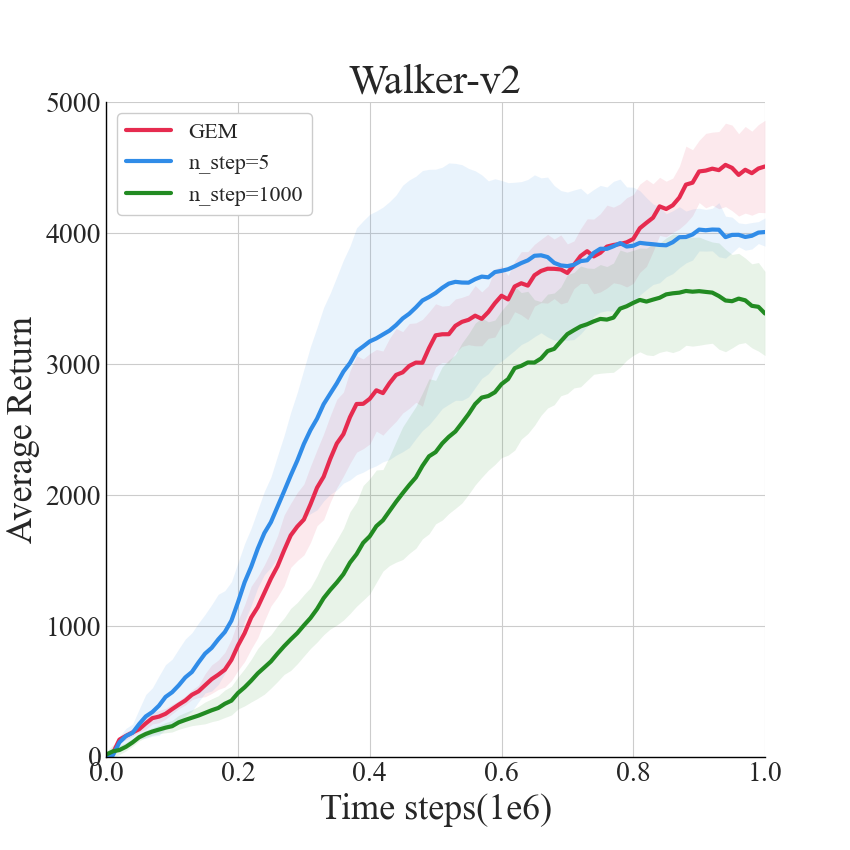}
    \end{subfigure}
    \begin{subfigure}[c]{0.3\textwidth}
     \centering
     \includegraphics[width=0.98\linewidth]{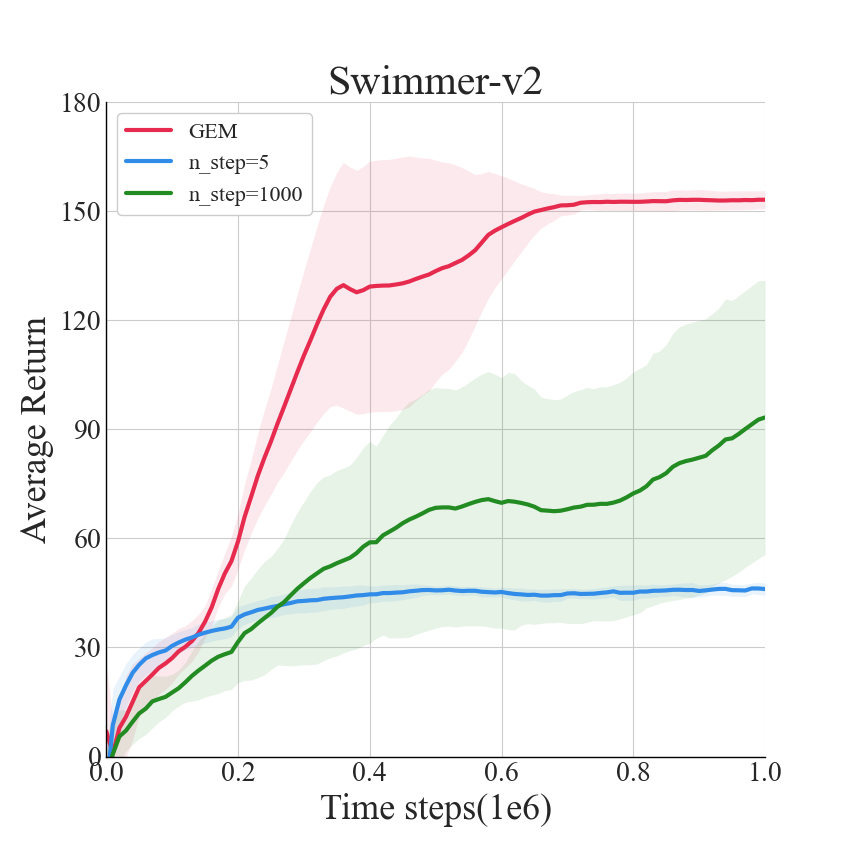}
    \end{subfigure}
    \caption{Comparison with simple n-step learning. The shaded region represents half a standard deviation of the average evaluation. Curves are smoothed uniformly for visual clarity.}
    \label{nstep}
   \end{figure*}

\begin{figure*}[hb]
    \centering
    \includegraphics[width=15cm]{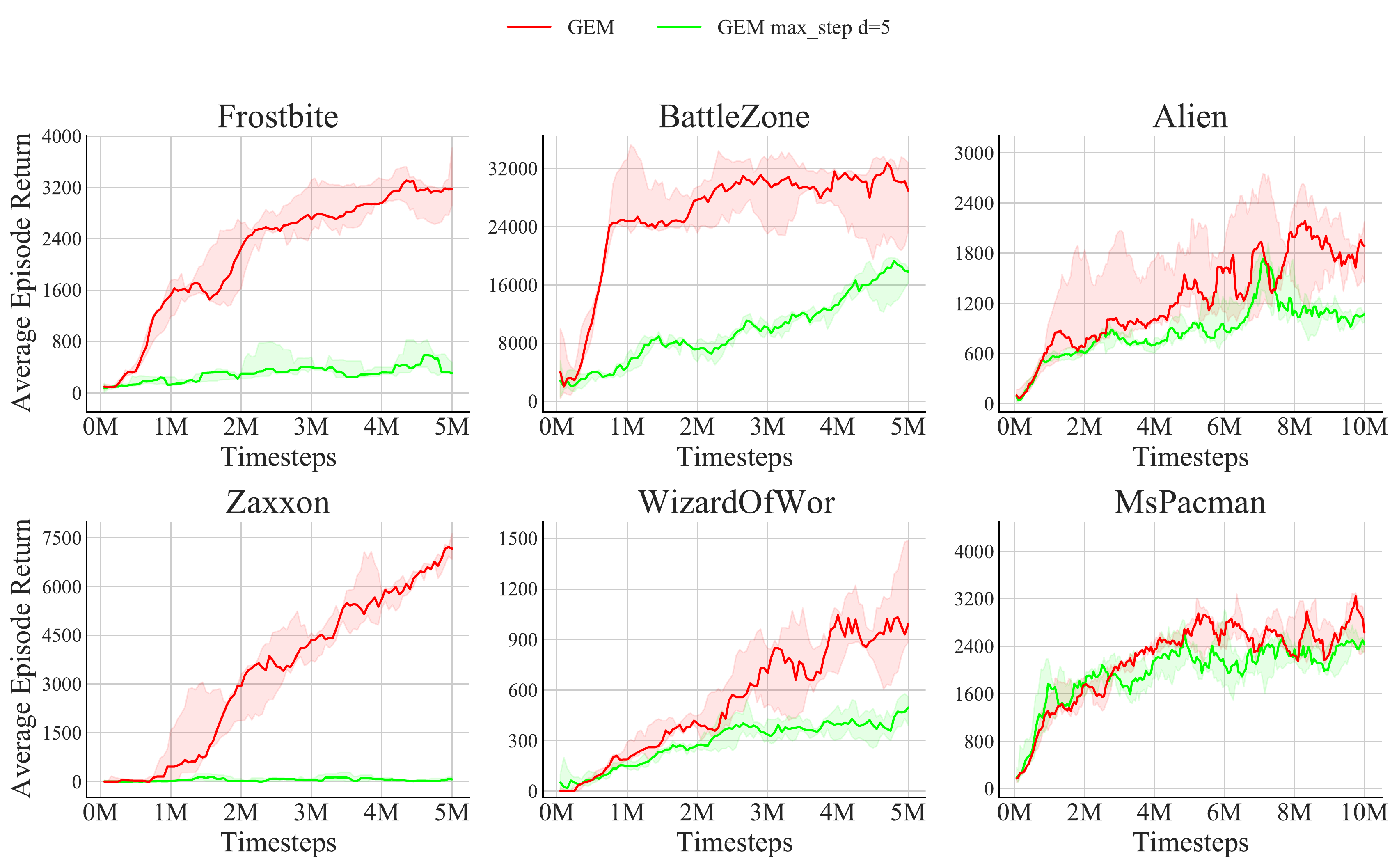}
    \caption{Ablation study on 6 Atari games. Limiting rollout lengths greatly affects the performance of GEM, which proves that GEM can use long rollout trajectories  effectively.}
    \label{atari_ablation}
\end{figure*}

\end{document}